\pdfoutput=1
\documentclass[twocolumn]{autart}  
\usepackage[utf8]{inputenc} 
\usepackage{subfig}
\usepackage{graphicx}          
\usepackage{cite}
\newcommand{\tc}[2]{\textcolor{#1}{#2}}
\usepackage[dvipsnames]{xcolor} %
\usepackage{amsmath,amssymb,amsfonts, mathabx}
\usepackage{algpseudocode}
\newcommand{\ones}{\mathbf 1}
\newcommand{\reals}{{\mathbb{R}}}

\newcommand{\naturals}{{\mathbb{N}}}

\newcommand{\argmin}{\mathop{\rm argmin}}
\newcommand{\argmax}{\mathop{\rm argmax}}

\newcommand{\norm}[1]{\left\lVert#1\right\rVert}
\newcommand{\mnorm}[1]{{\left\vert\kern-0.25ex\left\vert\kern-0.25ex\left\vert #1 
    \right\vert\kern-0.25ex\right\vert\kern-0.25ex\right\vert}}

\newcommand{\mc}{\mathcal}

\newcommand{\change}[1]{\tc{black}{#1}}
\begin{document}

\begin{frontmatter}

\title{Set-based Value Operators for Non-stationary and Uncertain Markov Decision Processes\thanksref{footnoteinfo}}

\thanks[footnoteinfo]{This research is partly funded by NSF grant CMMI-210563 and the University of Washington Aero\&Astro Condit fellowship. Corresponding author Sarah H.Q. Li. Email. sarahli@control.ee.ethz.ch.}

\author[UW]{Sarah H.Q. Li},  
\author[UP]{Assal\'{e} Adj\'{e}}, 
\author[ENAC]{Pierre-Lo\"{i}c Garoche},
\author[UW]{Beh\c{c}et A\c{c}ikme\c{s}e}

\address[UW]{Department of Aeronautics and Astronautics, University of Washington, Seattle, USA. (e-mail:\{sarahli, behcet\}@uw.edu).}  
\address[UP]{LAMPS, Universit\'{e} de Perpignan Via Domitia, Perpignan, France. (e-mail: assale.adje@univ-perp.fr).}             
\address[ENAC]{\'{E}cole Nationale de l'Aviation Civile, Universit\'{e} de Toulouse, Toulouse, France. (e-mail: Pierre-Loic.Garoche@enac.fr).}        


\begin{abstract}This paper analyzes finite state Markov Decision Processes (MDPs) with
uncertain parameters in compact sets and re-examines results from
robust MDP via set-based fixed point theory. 
To this end, we generalize the Bellman and policy evaluation operators
to contracting operators on the value function space and denote
them as \emph{value operators}. We lift these value operators to
act on \emph{sets} of value functions and denote them as \emph{set-based value operators}. We prove that the set-based value
operators are \emph{contractions} in the space of compact value function sets.
Leveraging insights from set theory, we generalize the rectangularity
condition in classic robust MDP literature
to a containment condition for all value operators, which
is weaker and can be applied to a larger set of parameter-uncertain
MDPs and contracting operators in dynamic programming. We prove that both the rectangularity condition and the containment condition sufficiently ensure that the set-based value operator's fixed point set contains its own extrema elements. For convex and compact sets of uncertain MDP parameters, we show equivalence between the classic robust value function and the supremum of the fixed point set of the set-based Bellman operator. Under dynamically changing MDP parameters in compact sets, we prove a set convergence result for value iteration, which otherwise may not converge to a single value function. Finally, we derive novel guarantees for probabilistic path planning problems in planet exploration and stratospheric station-keeping.
\end{abstract}
\end{frontmatter}

\section{Introduction}
Markov decision process (MDP) is a versatile model for decision making in stochastic environments and is widely used in trajectory planning~\cite{al2013wind}, robotics~\cite{van2015learning}, and operations research~\cite{doshi2005dynamic}. Given state-action costs and transition probabilities, finding an optimal policy of the MDP is equivalent to solving for the fixed point \emph{value function} of the corresponding Bellman operator. 

Many application settings of MDPs, including traffic light control, motion planning, and dexterous manipulation, deal with \emph{environmental non-stationarity}---dynamically changing MDP cost and transition probabilities due to external factors or the presence of interfering decision makers. This environmental non-stationarity corresponds to uncertainty in the MDP transition and cost parameters and differs from an MDP's internal stochasticity, which \change{is modeled by stationary stochastic dynamics} 
whose probability distributions do not change over time. 
\change{\begin{exmp}[Navigating in changing wind]\label{ex:wind_uncertainty} An autonomous aircraft navigates in a two-dimensional and time-varying wind field towards a non-stationary target, where the wind field varies between $N$ major patterns over time. The aircraft's transition probabilities are constructed from global averages of local wind observations, and the aircraft's objective is to reach the location of the non-stationary target, which is also affected by wind. 
If the wind pattern strictly switches between the discrete wind trends, then the transition uncertainty at state $s\in[S]$ is given by the set $\mc{P}_s = \{P^1_s, \ldots, P^N_s\}$. Similarly, the reachability cost of each state-action is also given by $\mc{C}_s = \{C^1_s,\ldots, C^N_s\}$. Collectively, we say that the MDP has time-varying parameters $\mc{M}_s =\{m^1_s, \ldots, m^N_s\} = \mc{C}_s\times\mc{P}_s$ at each state $s \in [S]$, where $m^i_s = (C^i_s, P^i_s)$.
\end{exmp}}

\change{Environmental non-stationarity differs from parameter uncertainty and yet is closely related. Parameter-uncertain dynamic programming assumes that the MDP has stationary yet unknown stochastic dynamics within a bounded set, and its performance can be bounded via worst-case performance in} robust MDP, risk-sensitive reinforcement learning, and zero-sum stochastic games---value functions that result from adversarial selections of the MDP parameters. 
\change{Under environmental non-stationarity, we assume that at every time instance, the MDP parameters are known but will vary in time unpredictably. Environmental non-stationarity is a better assumption than parameter uncertainty for scenarios such as  Example~\ref{ex:wind_uncertainty}, where the dynamics are stochastic, observable, non-adversarial, yet time-varying. } 

\change{Consider dynamic programming under environmental non-stationarity: at every time instance, the dynamic program is updated with respect to known but time-varying MDP parameters. Although interesting and highly relevant to many trajectory planning problems, this setting has no convergence guarantees. In fact, value iteration will most definitely diverge and can be demonstrated using simple examples. Does this mean that dynamic program has no convergence guarantees under environmental non-stationarity?}

\change{In this paper, we introduce a set-based framework to non-stationary MDPs that provides convergence guarantees under Hausdorff distance, and demonstrate that this set-based convergence also applies to parameter-uncertain MDPs and is related to dynamic programming robust dynamic programming.} 



\textbf{Contributions}.
\change{For environmental nonstationarity bounded by compact sets, } we propose the set-extensions of \emph{value operators}: a general class of contraction operators that \change{extends the Bellman operator and the policy evaluation operator. }
We prove the existence of compact \emph{fixed point sets} of the set-based value operators and show that the set-based value iteration converges. In a non-stationary Markovian environment, standard value iteration may not converge. However, we can show that the point-to-set distance of the resulting value function trajectory to the fixed point set always goes to zero in the limit. We derive a \emph{containment condition} that is sufficient for the fixed point sets to contain their own extremal elements. Within robust MDPs, we show that the containment condition generalizes the rectangularity condition, such that the optimal worst-case policy, or the robust policy, exists when the containment condition is satisfied. We then derive the relationship between the fixed point sets of 1) the set-based optimistic policy evaluation operator, 2) the set-based robust policy evaluation operator, and 3) the set-based Bellman operator.
Given a value operator and a compact MDP parameter uncertainty set, we present an algorithm that computes the bounds of the corresponding fixed point set and derive its convergence guarantees. Finally, we apply our results to the wind-assisted navigation of high altitude platform systems relevant to space exploration~\cite{wolf2010probabilistic} and show that our algorithms can be used to derive policies with better guarantees.


\textbf{Related research}.
MDP with parameter uncertainty is well studied in robust control and reinforcement learning. In control theory, the worst-case cost-to-go with respect to state-decoupled parameter uncertainties is derived via a minmax variation of the Bellman operator in~\cite{givan2000bounded,iyengar2005robust,nilim2005robust,wiesemann2013robust}. The cost-to-go under parameter uncertainty with coupling between states and time steps is similarly bounded in~\cite{mannor2016robust,goyal2022robust}. The effect of statistical uncertainty on the optimal cost-to-go is studied in~\cite{nilim2005robust,mannor2016robust,wiesemann2013robust,yang2017convex}. Recently, MDP with parameter uncertainty has gained traction in the reinforcement learning community due to the presence of uncertainty in real world problems such as traffic signal control and multi-agent coordination~\cite{kumar2020conservative,lecarpentier2019non,padakandla2020reinforcement}. Most RL research extends the minmax worst-case analysis to methods such as Q-learning and SARSA. Recently, methods for value-based RL using non-contracting operators have been investigated in~\cite{bellemare2016increasing}. 

As opposed to the worst-case approach to analyzing MDPs under parameter uncertainty, we do not assume adversarial MDP parameter selection. Instead, we derive a set of cost-to-gos that is invariant with respect to the compact parameter uncertainty sets for order-preserving, $\alpha$-contracting operators, a class that the Bellman operator belongs to. We continue from our previous work~\cite{li2021bounding}, in which we analyzed the set-based Bellman operator for cost uncertainty only.

\textbf{Notation}: 
A set of $N$ elements is given by $[N] = \{0, \ldots, N-1\}$. We denote the set of matrices of $i$ rows and $j$ columns with real (non-negative) entries as $\reals^{i \times j}$ ($\reals_+^{i\times j}$), respectively. Matrices and some integers are denoted by capital letters, $X$, while sets are denoted by cursive typeset $\mc{X}$. The set of all \emph{compact subsets} of $\reals^d$ is denoted by $\mc{K}(\reals^d)$. The column vector of ones of size $N \in \mathbb{N}$ is denoted by $\ones_N = [1, \ldots, 1]^T \in \reals^{N\times 1}$. The identity matrix of size $S$ is denoted by $I_S$. The simplex of dimension $S$ is denoted by
\begin{equation}\label{eqn:simplex_definition}
    \Delta_S = \{p \in \reals^S \ | \ \ones_{S}^\top  p = 1, \ p \geq 0\}.
\end{equation}
A vector $h \in \reals^S$ has equivalent notation $(h_1,\ldots, h_s)$, where $h_s$ is the value of $h$ in the $s^{th}$ coordinate, $s \in [S]$. 
Throughout the paper, $\norm{\cdot}$ denotes the infinity norm in $\reals^S$.
\section{Discounted infinite-horizon MDP}\label{sec:setup}
A \emph{discounted infinite-horizon finite state MDP} is given by $([S], [A], $ $ P, C, \gamma)$, where $\gamma \in (0, 1)$ is the discount factor, $[S] = \{1, \ldots, S\}$ is the \textbf{finite set of states} and $[A] = \{1, \ldots, A\}$ is the \textbf{finite set of actions}. Without loss of generality, assume that each action is admissible from each state $s \in [S]$. 

\textbf{MDP Costs}.  $C \in\reals^{S\times A}$ is the matrix encoding the MDP cost. Each $C_{sa} \in \reals$ is the cost of taking action $a\in[A]$ from state $s\in[S]$. We also denote the cost of all actions at state $s$ by $c_s = [C_{s1},\ldots,C_{sA}] \in \reals^A$, such that $C = [c_1, \ldots, c_S]^\top$.

\textbf{MDP Transition Dynamics}. The transition probabilities when action $a$ is taken from state $s$ are given by $p_{sa} \in \Delta_S$. Collectively, all possible transition probabilities from state $s \in [S]$ are given by the matrix  $\textstyle P_s =[p_{s1},\ldots,p_{sA}] \in \Delta_S^A \subset \reals^{S\times A}$, and all possible transition probabilities in the MDP are given by the matrix $P = [P_1,\ldots, P_S] \in \Delta_{S}^{SA} \subset \reals^{S\times SA}$.

\textbf{MDP Objective}. We want to minimize the expected cost-to-go, or the \textbf{value vector} $V \in \reals^S$, defined per state as 
\begin{equation}\label{eqn:policy_infinite_horizon_cost}
 \textstyle V_{s} :=\ \mathbb{E}_s \Big\{ \sum_{t = 0}^\infty \gamma^t C_{s^t a^t} \ | \ s^0 = s\Big\}, \ \forall \ s \in [S],
\end{equation}
where $\mathbb{E}_{s}\{\cdot\}$ is the expected value of the input with respect to initial state $s$, and ($s^t, a^t$) are the state and action at time $t$. 
\begin{rem}
Although \emph{value function} is the standard term for the expected cost-to-go, we use value vector in this paper to emphasize that the cost-to-go values of finite MDPs belong in a finite dimensional space.
\end{rem}

\textbf{MDP Policy}. \change{The decision maker controls the \emph{policy},} denoted as $\pi = [\pi_1,\ldots, \pi_S] \in \Delta_A^S $, where the $a^{th}$ element of $\pi_s \in \Delta_A$ is the conditional probability of action $a$ being chosen from state $s$. 
\change{Using the policy, we can minimize the value vector~\eqref{eqn:policy_infinite_horizon_cost} in a closed-loop fashion.
\begin{equation}\label{eqn:expected_infinite_horizon_cost}
 \textstyle V_{s}^\star :=\ \underset{\pi^t \in \Delta_{A}^S}{\min}\  \mathbb{E}_s \Big\{ \sum_{t = 0}^\infty \gamma^t C_{s^t \pi^t(s^t)} \ | \ s^0 = s\Big\}, \ \forall \ s \in [S],
\end{equation}
Under policy $\pi_s$, the expected immediate cost at $s$ is given by $c_s^\top \pi_s \in \reals$ and the expected transition probabilities from $s$ is given by $P_s \pi_s \in \Delta_S$.} 
\subsection{Value operators}
Solving an MDP is equivalent to finding the value vector and the associated policy that minimizes the objective~\eqref{eqn:expected_infinite_horizon_cost}. Typical solution methods utilize \emph{order preserving}~\cite[Def.3.1]{schroder2003ordered}, \emph{$\alpha$-contractive operators} whose fixed points are the optimal value vectors (e.g. Bellman operator~\cite[Thm.6.2.3]{puterman2014markov}, $Q$-value operator~\cite{melo2001convergence}).
\begin{defn}[$\alpha$-Contraction]\label{def:contractOp}
Let $(\mc{X}, d)$ be a metric space with metric $d$. The operator $H:\mc{X}\mapsto \mc{X}$ is an $\alpha$-contraction if and only if there exists $\alpha\in [0,1)$ such that
\begin{equation}\label{eqn:alpha_contraction}
d(H(V), H(V')) \leq \alpha d(V, V'), \quad \forall \ V, \ V' \in \mc{X}.   
\end{equation}
\end{defn}
\begin{defn}[Order Preservation]\label{def:order_preservation}
Let $(\mc{X}, \leq)$ be an ordered space with partial order $\leq$. The operator $H: \mc{X} \mapsto \mc{X}$ is order preserving if for all $V, V' \in \mc{X}$ such that $V\leq V'$, $H(V) \leq H(V')$. 
\end{defn}
These operators are typically locally Lipschitz in MDP parameter space.
\begin{defn}[$K(V)$-Lipschitz]\label{def:kv_lipschitz}
Let $(\mc{X}, d_{\mc{X}})$ be a metric space with metric $d_{\mc{X}}$ and $(\mc{Y}, d_{\change{\mc{Y}}})$ be a metric space with metric $d_{\mc{Y}}$. The operator $H:\mc{X}\times\mc{Y} \mapsto \mc{X}$ is \emph{$K(V)$-Lipschitz} with respect to $\mc{M} \subset \mc{Y}$ if for all $V \in \mc{X}$, there exists $K(V)\in \reals_+$ such that
\begin{equation}\label{eqn:kl}
    d_{\mc{X}}(H(V,m),H(V,m'))\leq K(V)d_{\mc{Y}}(m, m'), \ \forall m, m' \in \mc{M}.
\end{equation}
\end{defn}
\begin{rem}
The property $\alpha$-contraction is a special case of Lipschitz continuity when the input and output spaces are identical and the Lipschitz constant is less than $1$.
\end{rem}
To capture operators with these properties, we define a \textbf{value operator} that takes inputs: value vector, MDP cost, and MDP transition probability. The MDP cost and transition probability are selected from an MDP parameter set $\mc{M}$.
\begin{defn}[Value operator]\label{def:value_operator} 
Consider the operator $h$, 
\begin{equation}\label{eqn:value_operator}
h: \reals^S \times \mc{M} \mapsto \reals^S, \ \mc{M}\subseteq \reals^{S\times A} \times \Delta_{S}^{SA}.
\end{equation}
We say  $h$~\eqref{eqn:value_operator} is a \textbf{value operator} on $\reals^S\times\mc{M}$ if 
\begin{enumerate}
    \item  For all  $m \in \mc{M}$, $h(\cdot, m)$ is an $\alpha$-contraction in $\reals^S$.
    \item For all  $m \in  \mc{M}$, $h(\cdot, m)$ is order preserving in $\reals^S$.
    \item For all $V \in \reals^S$, $h(V, m)$ is $K(V)$-Lipschitz on $\mc{M}$.
\end{enumerate}
\end{defn}
\begin{rem}
While we only consider value operators whose input's first component is $\reals^S$, Definition~\ref{def:value_operator} and the subsequent results can be extended to the space of $Q$-value functions by swapping $\reals^S$ for $\reals^{SA}$ in Definition~\ref{def:value_operator}~\cite{melo2001convergence}. 
\end{rem}
\begin{figure}
    \centering
    \includegraphics[width=\columnwidth]{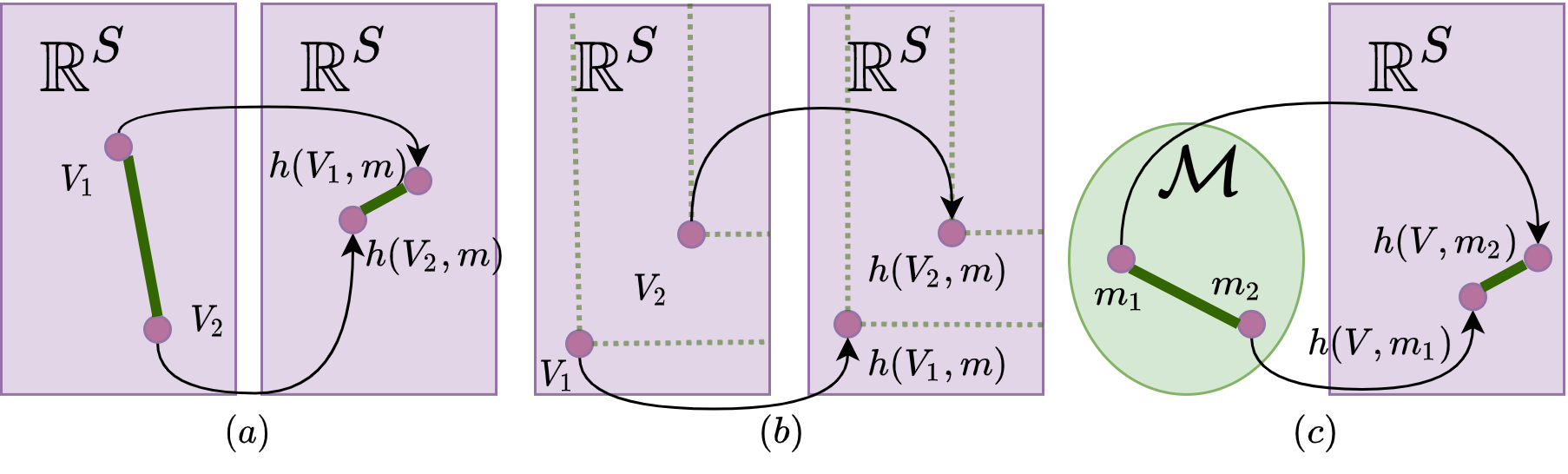}
    \caption{Illustration of the three value operator properties. (a) $\alpha$-contraction on $\reals^S$, (b) Order preservation on $\reals^S$, and (c) $K(V)$-Lipschitz in input space $\mc{M}$. }
    \label{fig:value_operator_visualization}
\end{figure}
An immediate consequence of the value operator $h$ being an $\alpha$-contractive and order-preserving operator on $\reals^S$ is that $h$ is continuous on $\reals^S\times \mc{M}$. 
\begin{lem}[Continuity]\label{lem:operator_continuity}
If $h$~\eqref{eqn:value_operator} is a value operator on $\reals^S\times\mc{M}$, $h$ is continuous on $\reals^S\times \mc{M}$.
\end{lem}
\change{See App.~\ref{app:2} for proof.}
Examples of value operators include the Bellman operator and the policy evaluation operators when the MDP cost and transition probability are input parameters rather than fixed parameters. 
\begin{defn}[Policy evaluation operator]\label{def:policy_operator}
Given a policy $\pi \in \Pi$, the vector-valued operator $g^\pi = (g^\pi_1,\ldots, g^\pi_S): \reals^{S}\times \reals^{S\times A}\times \Delta_{S}^{SA} \mapsto \reals^{S}$ is defined per state as
\begin{equation}\label{eqn:policy_operator} %
    g^\pi_s(V,  C, P) := c_s^\top \pi_s + \gamma \Big(P_s \pi_s\Big)^\top V, \ \forall s \in [S].
\end{equation}
Given $(C,P)$, $g^\pi(\cdot, C, P):\reals^{S}\mapsto\reals^S$ is a vector-valued operator whose fixed point is the expected cost-to-go of the MDP $([S], [A], C, P, \gamma)$ under $\pi$, denoted as $V^\pi(C, P)$~\cite[Thm.6.2.5]{puterman2014markov}. 
\begin{equation}~\label{eqn:fixed_points_policy}
    V^\pi(C, P) = g^\pi(V^\pi, C,P), \  V^\pi(C, P) \in \reals^S.
\end{equation}
When the context is clear, we denote $V^\pi(C,P)$ as $V^\pi$.
\end{defn}
\begin{defn}[Bellman operator]
\label{def:bellmanOp}The vector-valued operator $f = (f_1,\ldots, $ $f_S): \reals^{S}\times \reals^{S\times A}\times \Delta_S^{SA} \mapsto \reals^{S}$ is defined per each state as
\begin{equation}
f_s(V, C, P) :=  \inf_{\pi_s \in \Delta_A}\ g^\pi_s(V, C, P),\ \forall\, s\in [S].\label{eqn:bellman_operator}
\end{equation}
The corresponding optimal policy $\pi^\star = (\pi_1^\star,\ldots, \pi_s^\star)$ is defined per state as  $\pi_s^\star \in \argmin_{\pi_s} g^\pi_s(V, C, P)$~\eqref{eqn:bellman_operator}. 
One such policy is defined for all $(s, a) \in [S]\times[A]$ by
\begin{equation}\label{eqn:optimalPol}
    \pi^\star_{sa}  := \begin{cases} > 0 & a \in\underset{a \in [A]}{\argmin} \ C_{sa} + \gamma p_{sa}^\top V\\
    0 & \text{otherwise}
    \end{cases}, \sum_{a\in[A]} \pi^\star_{sa} = 1.
\end{equation}
where $\argmin_{a\in[A]}(h)$ is the set of minimizing actions for the  function $h$. An optimal policy in the form~\eqref{eqn:optimalPol} always exists for a discounted infinite horizon MDP~\cite[Thm 6.2.10]{puterman2014markov}. 
Given parameters $(C,P)$, $f(\cdot, C,P):\reals^{S}\mapsto\reals^S$ is a vector operator whose fixed point is the optimal cost-to-go for the MDP $([S], [A], P, C, \gamma)$, denoted as $V^{B}(C,P)$.
\begin{equation}~\label{eqn:fixed_points_bellman}
    V^{B}(C,P)  = f(V^{B}, C,P), \  V^{B}(C,P) \in \reals^S.
\end{equation}
When the context is clear, we denote $V^B(C,P)$ as $V^B.$
\end{defn}
We show that both~\eqref{eqn:policy_operator} and~\eqref{eqn:bellman_operator} are value operators. 
\begin{lem}
The Bellman operator~\eqref{eqn:bellman_operator} and the policy evaluation operators~\eqref{eqn:policy_operator} for all $\pi \in \Pi$ are value operators on $\reals^S \times \mc{M}$ where $\mc{M} \subseteq \reals^{S\times A} \times \Delta_{S}^{SA}$~\eqref{eqn:value_operator}.
\end{lem}
\change{See App.~\ref{app:3} for proof.}

\begin{rem}
Beyond the policy evaluation operator and the Bellman operator, many algorithms in reinforcement learning can be reformulated using value operators. For example, it's not difficult to show that the Q-learning operator~\cite{melo2001convergence} is a value operator on the vector space $\reals^{SA}$. 
\end{rem}

\section{Set-based value operators}\label{sec:contraction_operators}
\change{Motivated by stochastic and time-varying Markovian dynamics,} 
we now consider \change{set-based} value operators with respect to a compact set of  MDP parameters. 
We first introduce Hausdorff-type set distances. 
\begin{defn}[Point-to-set Distance]
The distance between a value vector and a set $\mc{V} \subseteq \reals^S$ is given by
\begin{equation}\label{eqn:distance_function}
   \textstyle W \mapsto d(W, \mc{V}) := \underset{V \in \mc{V}}{\inf}\norm{W - V}_{\change{\infty}}.
\end{equation}
\end{defn}
On the space of compact subsets of $\reals^{S}$, given by $\mc{K}(\reals^S)$,  the distance between value vector sets extends~\eqref{eqn:distance_function} and is given by the Hausdorff distance~\cite{henrikson1999completeness}.
\begin{defn}[Set-to-set Distance] The Hausdorff distance between two value vector sets $\mc{V}, \mc{W} \subseteq \reals^S$ is given by
\begin{equation}\label{eqn:hausdorff_distance}
    \begin{aligned}
 d_{\mc{K}}(\mc{V}, \mc{W}) := \max\left\{ \sup_{V \in \mc{V}} d(V,\mc{W}),\sup_{W \in \mc{W}} d(W,\mc{V})\right \}.
\end{aligned}
\end{equation}
\end{defn}
We use $(\mc{K}(\reals^{S}), d_{\mc{K}})$ to denote the metric space formed by the set of all compact subsets of $\reals^S$ under the Hausdorff distance $d_{\mc{K}}$. The induced Hausdorff space is complete if and only if the original metric space is complete~\cite[Thm 3.3]{henrikson1999completeness}. Therefore, $(\mc{K}(\reals^{S}), d_{\mc{K}})$ is a complete metric space.

For a value operator $h$~\eqref{eqn:value_operator}, we ask the following question: what is the \emph{set} of possible value vectors when the MDP has parameter \change{non-stationarity} given by $\mc{M}$? To resolve this, we define the set-based value operator $H$. 
\begin{defn}[Set-based Value Operator]\label{def:set_operator} The set-valued operator $H$ is induced by $h$ on $\reals^S\times\mc{M}$~\eqref{eqn:value_operator} and is defined as
\begin{equation}\label{eqn:general_set_operator}
  H(\mc{V}):=\left\{h(V,m)\ | \  (V, m) \in \mc{V}\times\mc{M}\right\} \subseteq \reals^S,
\end{equation}
where $\mc{V} \subseteq \reals^S$ is a subset of the value vector space.
\end{defn}
\begin{figure}
    \centering
    \includegraphics[width=0.9\columnwidth]{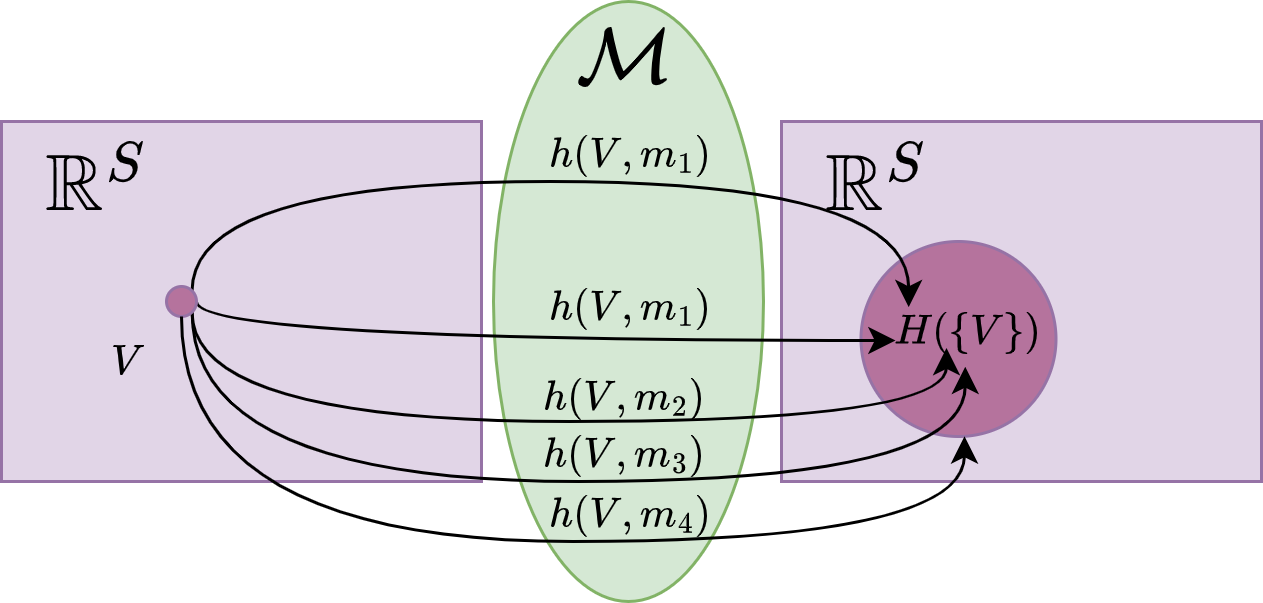}
    \caption{Illustration of the set-based operator $H(\mc{V})$ applied to the singleton set $\mc{V} = \{V\} \subset \reals^S$, we compute $h(V,m)$ for every parameter $m\in\mc{M}$ and collect the output $h(V, m)$, such that $H(\mc{V}) = \cup_{m\in\mc{M}} h(V, m)$.}
    \label{fig:my_label}
\end{figure}
We denote the set-based value operator induced by the Bellman operator~\eqref{eqn:bellman_operator} and policy evaluation operators~\eqref{eqn:policy_operator} as $F$ and $G^\pi$, respectively, such that for any value vector set $\mc{V} \subseteq \reals^S$,
\begin{equation}\label{eqn:set_bellman_ooperator}
    F(\mc{V}) := \left\{f(V, C, P) \ | \ (V, C, P) \in \mc{V}\times\mc{M}\right\},
\end{equation}
\begin{equation}\label{eqn:set_policy_operator}
    G^\pi(\mc{V}) := \left\{ g^\pi(V, C, P) \ | \ (V, C, P) \in \mc{V}\times\mc{M}\right\}, \ \forall \ \pi \in \Pi.
\end{equation}
The set-based Bellman operator \change{$F(\mc{V})$} is the union over all \change{one-step} \emph{optimal} value vectors, \change{which may result from different policies}, while \change{$G^\pi(\mc{V})$} is the union over all value vectors that result from \change{the same policy} $\pi$. 

We can ask the following question: is there a set of value vectors that is invariant with respect to $H$? Similar to the value operators $h$ from Definition~\ref{def:value_operator}, we can affirmatively answer this question by showing that $H$ is $\alpha$-contractive on $\mc{K}(\reals^S)$. 
\begin{thm}\label{thm:general_compact_contraction}
If $h$ is a value operator on $\reals^S\times \mc{M}$~\eqref{eqn:value_operator} and $\mc{M}$ is compact, then the induced set value operator $H$~\eqref{eqn:general_set_operator}  satisfies
\begin{enumerate}
    \item For all $\mc{V}\in \mc{K}(\reals^S)$, $H\left(\mc{V}\right)\in \mc{K}(\reals^{S})$;
    \item $H$ is an $\alpha$-contractive on $(\mc{K}(\reals^{S}), d_{\mc{K}})$~\eqref{eqn:hausdorff_distance} with a unique fixed point set $\mc{V}^\star$ given by
    \begin{equation}\label{eqn:fixed_point_set}
        \textstyle H(\mc{V}^\star) = \mc{V}^\star, \quad \mc{V}^\star \in \mc{K}(\reals^S);
    \end{equation}
    \item The sequence $\{\mc{V}^k\}_{k\in\mathbb{N}}$ where $\mc{V}^{k+1}=H(\mc{V}^k)$ converges to $\mc{V}^\star$ for any $\mc{V}^0\in \mc{K}(\reals^{S})$.
\end{enumerate}
In particular, these hold for $F$~\eqref{eqn:set_bellman_ooperator} and $G^\pi$~\eqref{eqn:set_policy_operator}, whose fixed point sets are denoted as $\mc{V}^B$ and $\mc{V}^\pi$, respectively.
\begin{equation}
    F(\mc{V}^B) = \mc{V}^B \in \mc{K}(\reals^S), \ G^\pi(\mc{V}^\pi) = \mc{V}^\pi \in \mc{K}(\reals^S), \ \forall \pi \in \Pi.
\end{equation}
\end{thm}
\begin{pf}
The first statement follows from Lemma~\ref{lem:operator_continuity}, since the image of a compact set by a continuous function is compact~\cite{rudin1964principles}.
Let us prove the second statement: for some $\beta\in (0,1)$, for all, $\mc{V},\mc{V}'\in \mc{K}(\reals^S)$:
\begin{equation*}
    \begin{aligned}
    & d_{\mc{K}}(H(\mc{V}),H(\mc{V}')) \\
    = &\max\left\{\sup_{\substack{V\in \mc{V}\\ m\in \mc{M}}}  d\left(h(V,m), H(\mc{V}')\right), \sup_{\substack{V'\in \mc{V}'\\ m'\in \mc{M}}} d(h(V',m'), H(\mc{V}))\right\} \\
\leq  & \beta d_{\mc{K}}(\mc{V},\mc{V}')
    \end{aligned}
\end{equation*}
Take $(V, m) \in \mc{V}\times\mc{M}$, then 
$d(h(V,m), H(\mc{V}'))\leq \inf_{V'\in \mc{V}'}\norm{h(V,m)-h(V',m)}_{\change{\infty}}\leq \alpha \inf_{V'\in \mc{V}'}\norm{V-V'}_{\change{\infty}}$ holds from the $\alpha$-contractive property of $h$. Finally,
\begin{equation*}
    \begin{aligned}
    \sup_{\substack{V\in \mc{V}\\ m\in \mc{M}}}d(h(V,m), H(\mc{V}'))\leq& \alpha \sup_{V\in  \mc{V}} \inf_{V'\in \mc{V}'} \norm{V-V'}_\infty\\
\leq & \alpha d_{\mc{K}}(\mc{V},\mc{V}') 
    \end{aligned}
\end{equation*}
We use the same technique to prove that 
\[\sup_{\substack{V'\in \mc{V}'\\ m'\in \mc{M}}}d(h(V',m'), H(\mc{V})) \leq \alpha d_{\mc{K}}(\mc{V},\mc{V}').\]
Finally,
$d_{\mc{K}}(H(\mc{V}),H(\mc{V}'))\leq \alpha d_{\mc{K}}(\mc{V},\mc{V}')$. From the Banach fixed point theorem and the completeness of $({\mc{K}}(\reals^S), d_{\mc{K}})$~\cite[Thm 3.3]{henrikson1999completeness}, $H$ has a unique fixed point \change{$\mc{V}^\star$} in ${\mc{K}}(\reals^S)$.

The third point is a consequence of the Banach fixed point theorem. 
Finally,  $f$ and $g^\pi$ are value operators~\eqref{eqn:value_operator} on $\reals^S\times\mc{M}$, therefore this theorem's statements apply. \qed
\end{pf}

\begin{rem}[Set-based value iteration]
An important consequence of
Theorem~\ref{thm:general_compact_contraction} is the existence of the \emph{set-based value iteration}, given by
\begin{equation}\label{eqn:set_based_value_iteration}
   \mc{V}^{k+1} = H(\mc{V}^k), \ \mc{V}^0 \in \mc{K}(\reals^S). 
\end{equation}
Analogous to standard value iteration, \eqref{eqn:set_based_value_iteration} is a sequence of value vector sets in  $\mc{K}(\reals^S)$ that converges to the fixed point set $\mc{V}^\star \in \mc{K}(\reals^S)$.
\end{rem}
\section{Properties of the fixed point set}\label{sec:properties_fixed_point_sets}
For the MDP parameters $(C,P)$, the fixed point of $h(\cdot, C, P)$ is typically meaningful for the corresponding MDP. For example, the fixed point of a policy evaluation operator $g^\pi(\cdot, C, P)$~\eqref{eqn:policy_operator} is the expected cost-to-go under policy $\pi$, and the fixed point of the Bellman operator $f(\cdot, C, P)$~\eqref{eqn:bellman_operator} is the minimum cost-to-go when $\pi$ can be freely chosen. In this section, we derive properties of the fixed point set $\mc{V}$ of $H$~\eqref{eqn:general_set_operator} in the context of non-stationary value iteration. 

\subsection{Non-stationary value iteration}
Given a value operator $h$ on $\reals^S\times \mc{M}$, we consider value iteration under a dynamic parameter uncertainty model discussed in~\cite{nilim2005robust}, where at every iteration, a new set of MDP parameters $m^k$ is chosen from $\mc{M}$ as 
\begin{equation}\label{eqn:time_varying_value_iteration}
    V^{k+1} = h(V^k, m^k), \ V^0 \in \reals^S, \ m^k \in \mc{M}, \forall k \in \mathbb{N}.
\end{equation}
In robust MDP literature~\cite{iyengar2005robust,nilim2005robust}, $m^k$ is modified by an adversarial opponent of the MDP decision maker such that~\eqref{eqn:time_varying_value_iteration} converges to a worst-case value vector. We consider a more general scenario in which $m^k$ is chosen from the closed and bounded set $\mc{M}$ without any probabilistic prior. In this scenario, convergence of $V^k$ in $\reals^S$ will not occur for all possible sequences of $\{m^k\}_{k\in\naturals}$. However, we can show convergence results on the set domain by leveraging our fixed point analysis of the set-based operator $H$~\eqref{eqn:general_set_operator}.
\begin{prop}\label{prop:param_varying} Let $\mc{V}^\star$ be the fixed point set of the set-based value operator $H$~\eqref{eqn:general_set_operator} induced by $h$ on $\reals^S\times \mc{M}$~\eqref{eqn:value_operator}. If the non-stationary value iteration~\eqref{eqn:time_varying_value_iteration} satisfies  $\{m^k\}_{k\in \mathbb N}\subset \mc{M}$, then the sequence $\{V^k\}_{k\in\mathbb{N}}$ defined by~\eqref{eqn:time_varying_value_iteration} satisfies
\begin{enumerate}
    \item  $\lim_{k\to +\infty} d(V^k, \mc{V}^\star) = 0$, 
    \item there exists a sub-sequence $\{V^{\varphi(k)}\}_{k\in\mathbb{N}}$ that converges to a point in $\mc{V}^\star$ as $\lim_{k\rightarrow\infty} V^{\varphi(k)} \in \mc{V}^\star$.
\end{enumerate}
\end{prop}

\begin{pf} Let $\{V^k\}_{k\in\mathbb N}$ be a sequence defined by $\mc{V}^0 = \{V^0\}$ and $\mc{V}^{k+1} = H(\mc{V}^k)$, where $H$~\eqref{eqn:general_set_operator} is the set operator induced by $h$ on $\reals^S\times\mc{M}$. We first show statement 1). From Theorem~\ref{thm:general_compact_contraction}, $\lim_{k\to\infty}\mc{V}^k$ converges to  $\mc{V}^\star$ in $d_{\mc{K}}$. Therefore, $0\leq d(V^k,\mc{V}^\star)=\inf_{y\in \mc{V}^\star} \norm{V^k-y}_\infty\leq \sup_{x\in \mc{V}^k}\inf_{y\in \mc{V}^\star} \norm{x-y}_\infty\leq d_{\change{\mc{K}}}(\mc{V}^k,\mc{V}^\star)\to 0$ as $k$ tends to $+\infty$. 

Next, for all $k\in\mathbb N$, there exists $N\in\mathbb N$ such that for all $n\geq N$, $d(V^n,\mc{V}^\star)\leq (k+1)^{-1}$. We define the strictly increasing function $\psi_1:\mathbb N\to \mathbb N$, such that $\psi_1(0)=0$ and for all $k\neq 0$, $\psi_1(k):=\min\{N > \psi_1(k-1): \forall n\geq N,\ d(V^n,\mc{V}^\star)<(k+1)^{-1}\}$. Then, for all $k\in\mathbb{N}^\star$, there exists $y^{\psi_1(k)}\in \mc{V}^\star$ such that $\norm{V^{\psi_1(k)}-y^{\psi_1(k)}}_{\change{\infty}}<(k+1)^{-1}$. As $\mc{V}^\star$ is compact, there exists $\psi_2:\mathbb N \to \mathbb N$ strictly increasing such that $(y^{\psi_1(\psi_2(k))})_k$ converges to some $y^\star\in \mc{V}^\star$~\cite[Thm 3.6]{rudin1964principles}. Finally, let $\varepsilon>0$, there exist $K_1,K_2\in \mathbb N$ such that for all $l\geq K_1$, $(\psi_2(l))^{-1}<\varepsilon/2$ and for all $l'\geq K_2$, $\norm{y^{\psi_1(\psi_2(l'))}-y^\star}_{\change{\infty}}<\varepsilon/2$. So, taking $k\geq \max\{K_1,K_2\}$, we have 
$\norm{V^{\psi_1(\psi_2(k))}-y^\star}_{\change{\infty}}\leq \norm{V^{\psi_1(\psi_2(k))}-y^{\psi_1(\psi_2(k))}}_{\change{\infty}}+\norm{y^{\psi_1(\psi_2(k))}-y^\star}_{\change{\infty}}\leq \varepsilon$
and $(V^{\psi_1(\psi_2(k))})_k$ converges to $y^\star\in \mc{V}^\star$.
\qed \end{pf}
In addition to containing all asymptotic behavior of value vector trajectories under time-varying value iteration, the fixed point set $\mc{V}$ also contains all fixed points of the value operator $h(\cdot, C, P)$ when $(C,P) \in \mc{M}$~\eqref{eqn:value_operator} are fixed.  
\begin{cor}
\label{cor:fixed_point_containment}
Let $h$~\eqref{eqn:value_operator} be a value operator on $\reals^S\times\mc{M}$ where $\mc{M}$ is compact. For all $m \in \mc{M}$, if $V=h(V,m) \in \reals^S$ and $\mc{V}^\star$ is the fixed point set of the induced set-based value operator $H$~\eqref{eqn:general_set_operator}, $V\in\mc{V}^\star$.
\end{cor}
\begin{pf}
We construct sequence $\{V^k\}$ where $V^{k+1}=h(V^k, m)$ and $V^0 = V$. Then $V^k = V$ for all $k \in \mathbb{N}$. From the second point of Proposition~\ref{prop:param_varying}, $V \in \mc{V}^\star$ follows. 
\qed \end{pf}
Going further, we can bound the transient behavior of~\eqref{eqn:time_varying_value_iteration} when $V^0$ is an element of the fixed point set $\mc{V}^\star$. 
\begin{cor}[Transient behavior]\label{cor:general_iteration_transient}
Let $\mc{V}^\star$ be the fixed point of the set-based value operator $H$~\eqref{eqn:general_set_operator} induced by $h $ on $\reals^S\times\mc{M}$. If $\mc{M}$ is compact and $V^0 \in \mc{V}^\star$, then the sequence generated by~\eqref{eqn:time_varying_value_iteration} satisfies $\{V^k\}_{k\in\mathbb{N}} \subseteq \mc{V}^\star$.
\end{cor}
\begin{pf}
As a fixed point set of $H$~\eqref{eqn:general_set_operator}, $\mc{V}^\star$~\eqref{eqn:fixed_point_set} satisfies $\mc{V}^\star = H(\mc{V}^\star)$, then the following is true by definition of $H$: if $V^k \in \mc{V}^\star$, then $V^{k+1} = h(V^k, m^k) \in \mc{V}^\star$. If $V^0 \in \mc{V}^\star$, then $\{V^k\}_{k\in\mathbb{N}} \subseteq \mc{V}^\star$ follows by induction. 
\qed \end{pf}
\begin{rem}
Proposition~\ref{prop:param_varying} and Corollary~\ref{cor:general_iteration_transient} bound the asymptotic and transient behavior of the sequence $\{h(V^k, m^k)\}_{k\in\naturals}$ generated from~\eqref{eqn:time_varying_value_iteration}, regardless of the convergence of the value vector sequence. This is a more general result then the classic convergence results for MDPs and robust MDPs.
\end{rem}
\begin{rem}
Corollary~\ref{cor:general_iteration_transient} also implies that $\mc{V}^\star$ is invariant with respect to the non-stationary value iteration~\eqref{eqn:time_varying_value_iteration}, and may prove useful in the analysis and design of MDPs with known parameter uncertainties.
\end{rem}
\subsection{Bounds of the fixed point set}
In Theorem~\ref{thm:general_compact_contraction}, the compactness of $\mc{M}$ implied the compactness of $\mc{V}^\star$. This relationship carries over to the supremum and infimum elements of $\mc{M}$ and $\mc{V}^\star$---i.e., if $\mc{M}$ satisfies Assumption~\ref{assum:containment_condition} with respect to $h$, then $\mc{V}^\star$ contains its own supremum and infimum elements. 

\textbf{Greatest and least elements}. We define the supremum and infimum elements of a value vector set $\mc{V} \in \mc{K}(\reals^S)$ element-wise as follows, 
\begin{equation}\label{eqn:sup_inf_element_def}
    \overline{V}_s := \sup_{V \in \mc{V}} V_s,\  \underline{V}_s := \inf_{V \in \mc{V}} V_s, \forall \ s\in [S].
\end{equation}
\begin{figure}[ht]
    \centering
    \includegraphics[width=0.9\columnwidth]{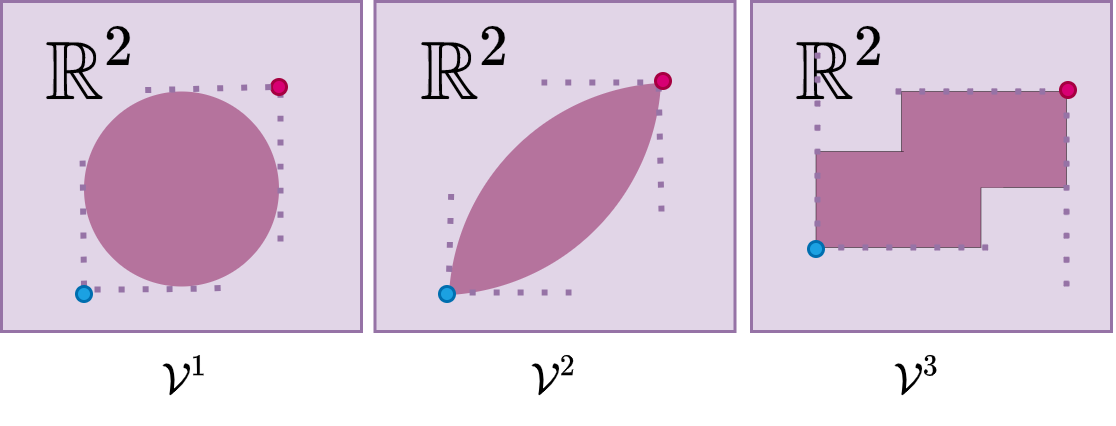}
    \caption{The greatest least bounds of three different value function sets $\mc{V}^i \in \reals^2$, where $(0,0)$ the origin is located on the lower left. Note that $\mc{V}^2$ and $\mc{V}^3$ contains their own greatest and least elements, but $\mc{V}^1$ does not. In $\mc{V}^1$, the coordinate-wise greatest and least elements are achieved by some element in $\mc{V}^1$ but not at the same time. }
    \label{fig:greatest_least_bounds}
\end{figure}

If a set $\mc{V}\subseteq \reals^{S}$ is compact, the projection of $\mc{V}$ on each state $s$ is compact. Then, the coordinate-wise supremum and infimum values for each state $s$ are achieved by $\mc{V}$. However in general, no single element of the set $\mc{V}$ may simultaneously achieve the minimum over all states---i.e., $\overline{V}(\underline{V})$ may not be an element of $\mc{V}$. 
This is illustrated in Figure~\ref{fig:greatest_least_bounds}. 


Given $h$ and parameter uncertainty set $\mc{M}$, we wish to 1) bound the supremum and infimum elements of the fixed point set $\mc{V}^\star$~\eqref{eqn:fixed_point_set} and 2) derive sufficient conditions for when they are elements of $\mc{V}^\star$. 
To facilitate bounding $\mc{V}^\star$, we introduce the following bound operators. 
\begin{defn}[Bound Operators] The bound operators induced by the value operator $h$ on $\reals^S\times\mc{M}$ are coordinate-wise defined at each $s \in [S]$ as 
\begin{equation}\label{eqn:infsup_parameter_contractions}
\underline{h}_s(V) = \inf_{m \in \mc{M}} h_s(V, m), \ \overline{h}_s(V) = \sup_{m\in \mc{M}} h_s(V, m).
\end{equation}
\end{defn}
\begin{figure}[ht]
    \centering
    \includegraphics[width=0.65\columnwidth]{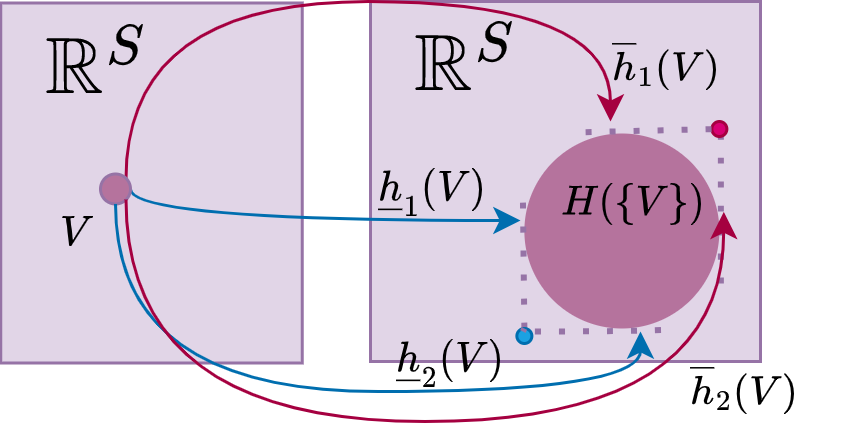}
    \caption{We visualize the bound operator for $H(\mc{V})$ for a given value operator $h$ on $\reals^S\times \mc{M}$. The input set $\mc{V}$ is a singleton $\{V\}$ in $\reals^2$. Here, because $\underline{h}_1$ and $\underline{h}_2$are reached for two different parameters $m \in \mc{M}$, the resulting $\underline{h}(V)$ lies outside of the fixed point set.}
    \label{fig:bound_value_operators}
\end{figure}
We want to bound the fixed point set $\mc{V}$ of the set-based value operator $H$~\eqref{eqn:general_set_operator} by the bound operators $\underline{h}/\overline{h}$~\eqref{eqn:infsup_parameter_contractions}. First we show that $\underline{h}/\overline{h}$ are themselves $\alpha$-contractive and order preserving on $\reals^S$.
\begin{lem}[$\alpha$-Contraction]\label{lem:parametrized_contraction}
If $h$~\eqref{eqn:value_operator} is a value operator on $\reals^{S}\times\mc{M}$ and $\mc{M}$ is compact, 
then $\underline{h}$ and $\overline{h}$~\eqref{eqn:infsup_parameter_contractions} are $\alpha$-contractions with fixed points $\underline{X}, \overline{X}$, respectively. 
\begin{equation}\label{eqn:infsup_fixed_points}
    \textstyle \overline{h}(\overline{X}) = \overline{X}, \quad \underline{h}( \underline{X}) = \underline{X}, \ \underline{X}, \overline{X} \in \reals^S. 
\end{equation}
\end{lem}
\begin{pf}
From Lemma~\ref{lem:operator_continuity}, $h$ is continuous and $\mc{M}$ is compact, then for all $X, Y \in \reals^S$, there exists $\hat{m}(s) \in \mc{M}$ such that $\underline{h}_s(Y) = h_s(Y, \hat{m}(s))$ and $\underline{h}_s(X) \leq h_s(X, \hat{m}(s))$. We upper-bound $\underline{h}_s(X) - \underline{h}_s(Y)$ by $h_s(X, \hat{m}(s)) - h_s(Y, \hat{m}(s))$, and use the $\alpha$-contraction property of $h$ to derive
\begin{align*}
    \underline{h}_s(X) - \underline{h}_s(Y) & \leq |h_s(X, \hat{m}(s)) - h_s(Y, \hat{m}(s))| \\
    & \leq \alpha \norm{X - Y}_\infty.
\end{align*}
Since $X$ and $Y$ are arbitrarily ordered, we conclude that $\norm{\underline{h}(X) - \underline{h}(Y)}_\infty \leq \alpha\norm{X - Y}_\infty$. The proof for $\overline{h}$ follows a similar reasoning and takes $\hat{m}(s) = \change{\argmax}_{m\in\mc{M}} h_s(X, m)$. 
The existence of $\underline{X}(\overline{X})$ follows from applying Banach's fixed point theorem. 
\qed \end{pf}
\begin{lem}[Order Preservation]\label{lem:parameterized_order_preserving}
The bound operators  $\underline{h}$ and $\overline{h}$~\eqref{eqn:infsup_parameter_contractions} are order-preserving on $\reals^S$ (Definition~\ref{def:order_preservation}).
\[\forall \ U, V \in \reals^S, \ U\leq V \Rightarrow \underline{h}(U) \leq \underline{h}(V), \quad \overline{h}(U) \leq \overline{h}(V).\] 
\end{lem}
\begin{pf}
The lemma statement follows directly from the fact that order preservation is conserved through composition with $\inf$ and $\sup$. If $h(U, m) \leq h(V, m)$, then $\inf_{m\in\mc{M}} h(U,m) \leq \inf_{m\in\mc{M}} h(V, m)$. A similar argument follows for $\overline{h}(\cdot) = \sup_{m\in\mc{M}}h(\cdot, m)$. 
\qed \end{pf}
We show that the fixed points $\underline{X}$ and $\overline{X}$ bounds the fixed point set $\mc{V}^\star$ of the set-based value operator $H$~\eqref{eqn:general_set_operator}.
\begin{thm}[Bounding fixed point sets]\label{thm:bound_containment_inexact}
If $h$~\eqref{eqn:value_operator} is a value operator on $\reals^S\times\mc{M}$ and $\mc{M}$ is compact, 
\begin{equation}\label{eqn:bound_general}
    \underline{X} \leq  V \leq \overline{X},  \ \forall  \ V \in \mc{V}^\star,
\end{equation}
where $\underline{X}$ and $\overline{X}$~\eqref{eqn:infsup_fixed_points} are the fixed points of the bound operators $\underline{h}$ and $\overline{h}$~\eqref{eqn:infsup_parameter_contractions}, and $\mc{V}^\star$ is the fixed point set of the set-based value operator $H$~\eqref{eqn:general_set_operator} induced by $h$~\eqref{eqn:value_operator} on $\reals^S\times\mc{M}$. 
\end{thm}
\begin{pf} 
For $\mc{V}^0=\{\underline{X},\overline{X}\}$ and $\mc{V}^{k+1} = H(\mc{V}^{k})$~\eqref{eqn:set_based_value_iteration}, we first show 
\begin{equation}\label{eqn:proof_bound_0}
   \textstyle \underline{X} \leq  V \leq \overline{X}, \ \forall \ V \in \mc{V}^k, 
\end{equation}
via induction. Suppose that~\eqref{eqn:proof_bound_0} is satisfied for $\mc{V}^k$. The order preserving property of $h(\cdot, m)$ implies that
$h(\underline{X}, m) \leq  h(V, m)  \leq h(\overline{X}, m)$ holds for all $(V, m) \in \mc{V}^k\times \mc{M}$. We take the infimum and supremum over $h(\underline{X}, m)$ and $h(\overline{X}, m)$, respectively, to show that for all $(V, m) \in \mc{V}^k\times \mc{M}$ and $s \in [S]$,
\[ \inf_{m'\in\mc{M}} h_s(\underline{X}, m') \leq  h_s(V, m) \leq \sup_{m'\in\mc{M}} h_s(\overline{X}, m').\]
Since $\underline{X}$ and $\overline{X}$ are the fixed points of $\inf_{m'\in\mc{M}} h_s(\cdot, m')$ and $\sup_{m'\in\mc{M}} h_s(\cdot, m')$ for all $s \in [S]$, respectively, we conclude that~\eqref{eqn:proof_bound_0} holds for $\mc{V}^{k+1}$. 

Next, we show that $\underline{X}$ and $\overline{X}$ bounds the fixed point set $\mc{V}^\star$ for the $h$-induced operator $H$~\eqref{eqn:general_set_operator}. From Lemma~\ref{basicdh}, we know that for all $V \in \mc{V}^\star$, there exists a strictly increasing sequence $\phi:\mathbb{N}\mapsto\mathbb{N}$ and corresponding value vectors $\{W^{\phi(n)}\}$ such that $\lim_{n\to\infty}W^{\phi(n)} = V$ and $W^{\phi(n)} \in \mc{V}^{\phi(n)}$ for the sequence of value vector sets generated from $\mc{V}^0=\{\underline{X},\overline{X}\}$. Since $\underline{X} \leq W^{\phi(n)}\leq \overline{X}$ holds for all $n$, we conclude~\eqref{eqn:bound_general} holds. 
\qed \end{pf}

\section{Revisiting robust MDP}
We re-examine robust MDP with the set-theoretical analysis in this section, and
show that Assumption~\ref{assum:containment_condition} generalizes the rectangularity assumption made in robust MDPs, thus enabling robust dynamic programming techniques to be available to a wider class of MDP problems and contraction operators. 

Recall the optimistic value vector $W^o \in \reals^S$ and robust value vectors $W^r \in \reals^S$ of a discounted MDP $([S],[A], C, P, \gamma)$ from~\cite{iyengar2005robust,nilim2005robust} as the fixed points of the following operators. \begin{equation}\label{eqn:optimistic_value_function}
    W^o_s = \min_{\pi_s\in\Delta_A}\min_{(C, P) \in \mc{M}} g^\pi_s(W^o, C, P), \forall s \in [S]
\end{equation}
\begin{equation}\label{eqn:robust_value_function}
    W^r_s = \min_{\pi_s\in\Delta_A}\max_{(C, P) \in \mc{M}} g\change{_s}^\pi(W^r, C, P), \forall s \in [S]
\end{equation}
The optimistic policy $\pi^o$ and robust policy $\pi^r$ are the optimal policies corresponding to~\eqref{eqn:optimistic_value_function} and~\eqref{eqn:robust_value_function}, respectively.
\begin{equation}\label{eqn:optimistic_policy}
    \pi^o_s \in \argmin_{\pi_s\in \Delta_A}\min_{(C, P) \in \mc{M}} g_s^\pi(W^o,  C, P), \forall s \in [S]
\end{equation}
\begin{equation}\label{eqn:robust_policy}
    \pi^r_s \in \argmin_{\pi_s\in \Delta_A}\max_{(C, P) \in \mc{M}} g_s^\pi(W^r, C, P), \forall s \in [S]
\end{equation}
For readability, we denote the policy evaluation operator\change{~\eqref{eqn:policy_operator}} under $\pi^o$ as $g^o$ and the policy evaluation operator\change{~\eqref{eqn:policy_operator}} under $\pi^r$ as $g^r$. 

When $\mc{M}$ is $(s,a)$-rectangular~\eqref{eqn:sa_rectangular}, the set of policies satisfying~\eqref{eqn:optimistic_policy} and~\eqref{eqn:robust_policy} are non-empty and includes deterministic policies~\cite[Thm 3.1]{iyengar2005robust}. When $\mc{M}$ is $s$-rectangular and convex, the set of policies satisfying~\eqref{eqn:robust_policy} is non-empty but may be mixed~\cite[Thm 4]{wiesemann2013robust}. When $\mc{M}$ is convex, we show that policies~\eqref{eqn:optimistic_policy} and~\eqref{eqn:robust_policy} exist. 
\begin{prop}\label{prop:containment_to_value_existence}
If the MDP parameter set $\mc{M}$ is compact and convex, then 
\begin{enumerate}
    \item $W^o$~\eqref{eqn:optimistic_value_function} and $W^r$~\eqref{eqn:robust_value_function} exist and satisfy $\overline{f}(W^r) = W^r, \ \underline{f}(W^o) = W^o$, where $\overline{f}$ and $\underline{f}$~\eqref{eqn:infsup_parameter_contractions} are the bound operators of the Bellman operator~\eqref{eqn:bellman_operator}.
    \item $\pi^o$~\eqref{eqn:optimistic_policy} and $\pi^r$~\eqref{eqn:robust_policy} exist. 
\end{enumerate}
\end{prop}
\begin{pf}
Recall the Bellman operator $f$~\change{\eqref{eqn:bellman_operator}}. When $\mc{M}\times\Delta_A$ is compact, the formulation of the fixed point of $\underline{f}$~\eqref{eqn:infsup_parameter_contractions} is equivalently given by
\begin{equation}\label{eqn:prop3_proof_0}
\change{\underline{f}_s}(\underline{X}) = \min_{(C,P)\in \mc{M}} \min_{\pi_s \in \Delta_A} g^\pi_s(\underline{X}, C,P), \ \forall s \in [S].  
\end{equation}
We note that~\eqref{eqn:prop3_proof_0} is identical to
the formulation of $W^o$~\eqref{eqn:optimistic_value_function}. Therefore, $W^o = \underline{X}$ is the fixed point of $\underline{f}$.
When $\mc{M}$ is compact, $W^o$ exists due to Lemma~\ref{lem:parametrized_contraction}. From~\eqref{eqn:optimistic_policy}, $\pi^o_s$ is the optimal argument of $g^\pi_s(W^o, C,P)$, a continuous function in $\pi_s , C, P$ minimized over compact sets $\Delta_A \times \mc{M}$ for all $s \in [S]$. Therefore $\pi^o_s$ exists. Since $\pi^o = (\pi^o_1,\ldots, \pi^o_S)$, the optimal $\pi^o \in\Pi$ exists.

For the robust scenario: when $\mc{M}$ is compact, the fixed point of $\overline{f}$~\eqref{eqn:infsup_parameter_contractions}, $\overline{X}$, exists from Lemma~\ref{lem:parametrized_contraction} and is  given by
\begin{equation}\label{eqn:prop3_proof0}
    \overline{X}_s = \max_{(C, P) \in \mc{M}} \min_{\pi_s\in\Delta_A} g^\pi_s(\overline{X}, C, P), \ \forall s \in [S].
\end{equation} 
The function $g_s^\pi(\overline{X}, C, P)$ is concave in $(C,P)$ and convex in $\pi$. If $\mc{M}$ is convex, then we apply the minimax theorem~\cite{neumann1928theorie} to switch the order of $\min$ and $\max$ in~\eqref{eqn:prop3_proof0} to derive 
\begin{equation}\label{eqn:prop3_proof1}
    \overline{X}_s =  \min_{\pi_s \in \Delta_A}\max_{(C, P) \in \mc{M}} g^\pi_s(\overline{X}, C, P), \ \forall s \in [S].
\end{equation}
Equation~\eqref{eqn:prop3_proof1} is identical to~\eqref{eqn:robust_value_function}, therefore $W^r = \overline{X}$ and exists by Lemma~\ref{lem:parametrized_contraction}. In~\eqref{eqn:prop3_proof1}, $\max_{(C, P) \in \mc{M}}g^\pi_s(\overline{X}, C, P)$ is piece-wise linear in $\pi_s$ and $\Delta_A$ is compact for all $s \in [S]$, thus $\argmin_{\pi_s\in\Delta_A}$ $\max_{(C, P) \in \mc{M}}g^\pi_s(\overline{X}, C, P)$ is non-empty. Finally\change{, } since $\pi^r = (\pi^r_1,\ldots, \pi^r_S)$, $\pi^r$ exists. 
\qed \end{pf}
\begin{rem}
Since $\max_{(C, P) \in \mc{M}}g^\pi_s(\overline{X}, C, P)$ is piecewise linear in $\pi_s$, the optimal $\pi_s^r$ is mixed policy in general. This is consistent with the results in~\cite{wiesemann2013robust}.
\end{rem}
Proposition~\ref{prop:containment_to_value_existence} generalizes the results from~\cite{wiesemann2013robust} to show that~\eqref{eqn:robust_value_function} exists when $\mc{M}$ is compact and convex instead of $s$-rectangular and convex. From Theorem~\ref{thm:bound_containment_inexact}, $W^o$ and $W^r$ are the fixed points of \change{the bound operators $\underline{g}^{\pi^o}$ and $\overline{g}^{\pi_r}$~\eqref{eqn:bound_general}, respectively.} They become infimum and supremum elements when $\mc{M}$ satisfies Assumption~\ref{assum:containment_condition} with respect to $g^o$ and $g^r$. We explicitly derive this result next. First, we introduce some notations: let $G^o = G^{\pi_o}$, the fixed point of $G^o$ be $\mc{V}^o$,  $G^r = G^{\pi^r}$, and the fixed point of $G^r$ be $\mc{V}^r$.
\begin{equation}\label{eqn:optimistic_fixed_point_set}
    \mc{V}^o = \{g^{o}(V, C, P) \ | \ (C, P) \in \mc{M}, V \in \mc{V}^o\},
\end{equation}
\begin{equation}\label{eqn:robust_fixed_point_set}
    \mc{V}^r = \{g^{r}(V, C, P) \ | \ (C, P) \in\mc{M}, V \in \mc{V}^r\}.
\end{equation}
Additionally, the supremum elements of $\mc{V}^o$ and $\mc{V}^r$ are $\overline{V}^o$ and $\overline{V}^r$ respectively and the infimum elements are $\underline{V}^o$ and $\underline{V}^r$, respectively. 
\begin{equation}\label{eqn:robust_bound_defs}
    \underline{V}_s^r = \min_{V \in \mc{V}^r} V_s, \  \overline{V}_s^r  = \max_{V \in \mc{V}_{r}} V_s, \ \forall s \in [S]. 
\end{equation}
\begin{equation}\label{eqn:optimistic_bound_defs}
    \underline{V}_s^o = \min_{V \in \mc{V}^o} V_s, \  \overline{V}_s^o = \max_{V \in \mc{V}^o} V_s, \ \forall s \in [S]. 
\end{equation}
We compare these with the fixed point set of the Bellman operator, $\mc{V}^B =\{\min_{\pi}g^\pi(V, C, P) \ | \ (C, P) \in \mc{M}, V \in \mc{V}^B\}$~\eqref{eqn:fixed_point_set}, denoted by $\overline{V}^B$ and $\underline{V}^B$ as 
\begin{equation}\label{eqn:bellman_bound_defs}
    \underline{V}_s^B = \min_{V \in \mc{V}^B} \change{V_s}, \  \overline{V}_s^B = \max_{V \in \mc{V}^B} V_s, \ \forall s \in [S]. 
\end{equation}
\section{Fixed-point set containing its infimum/supremum}
We make the following assumption on the MDP parameter set $\mc{M}$ with respect to $h$.
\begin{assum}[Containment condition]\label{assum:containment_condition} The MDP parameter set $\mc{M}$ satisfies the containment condition with respect to $h$ if $\mc{M}$ is compact and for all $V \in \reals^S$, 
\begin{equation}\label{eqn:containment_args_condition}
    \bigcap_{s\in [S]} \argmin_{m\in\mc{M}} h_s(V, m) \neq \emptyset, \ \bigcap_{s\in [S]}\argmax_{m\in\mc{M}} h_s(V, m) \neq \emptyset.
\end{equation}
\end{assum}
\begin{figure}[ht]
    \centering
    \includegraphics[width=\columnwidth]{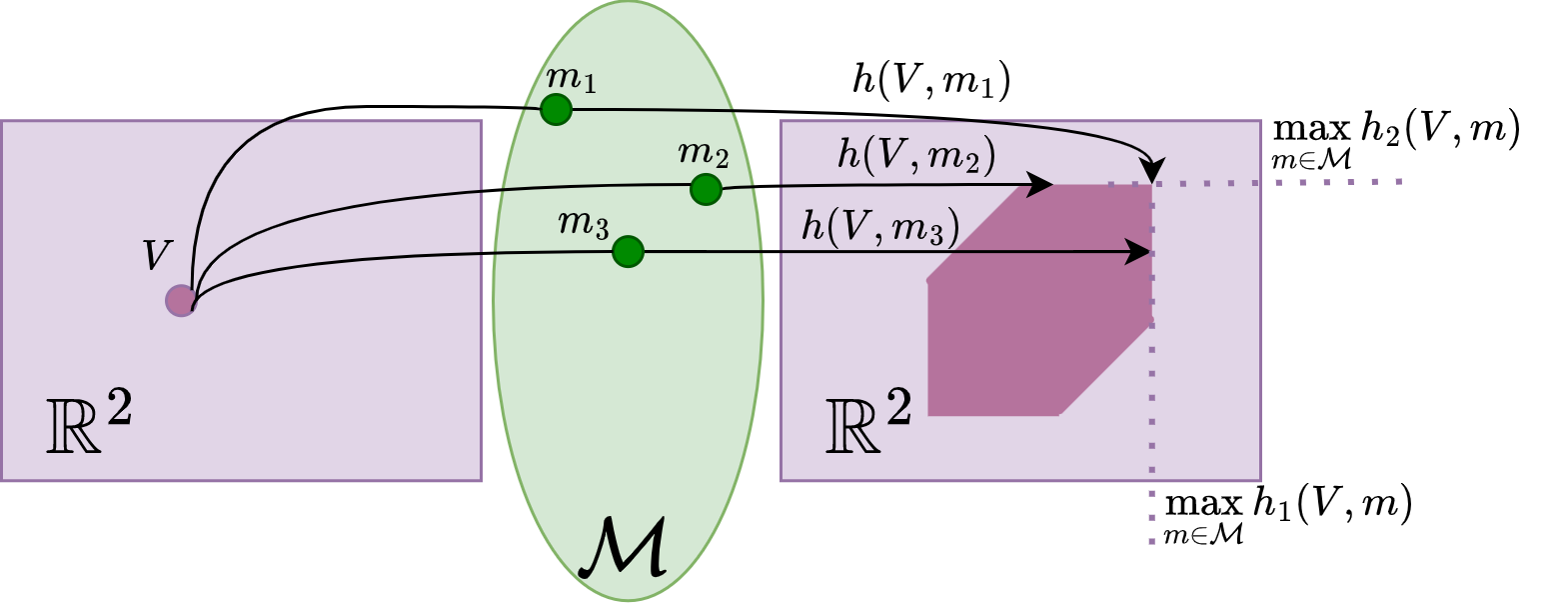}
    \caption{We illustrate $\argmax_{m\in\mc{M}} h_s(V, m)$ for a value operator $h$ when $S =2$. Here, $\argmax_{m\in\mc{M}} h_1(V, m) = \{m_2, m_3\}$, $\argmax_{m\in\mc{M}} h_2(V, m) = \{m_1, m_2\}$. Therefore, $m_2$ is the common parameter that achieves $\max_{m\in\mc{M}} h_s(V,m)$ for all $s \in [S]$.}
    \label{fig:containment_condition}
\end{figure}
\begin{rem}
Assumption~\ref{assum:containment_condition} is an $h$-dependent condition imposed on the structure of $\mc{M}$, and is independent of $\mc{M}$'s convexity and connectivity. 
\end{rem}
\subsection{Containment-satisfying MDP parameter sets}
Assumption~\ref{assum:containment_condition} restricts the structure of $\mc{M}$ with respect to the value operator $h$. Thus whether or not $\mc{M}$ satisfies Assumption~\ref{assum:containment_condition} must always be determined with respect to the operator $h$. With respect to the Bellman operator $f$~\eqref{eqn:bellman_operator} and the policy evaluation operators $g^\pi$~\eqref{eqn:policy_operator}, the following conditions in robust MDP are sufficient to satisfy Assumption~\ref{assum:containment_condition}.
\begin{defn}[$(s,a)$-rectangular sets~\cite{iyengar2005robust,nilim2005robust}]\label{def:sa_rectangular}
The uncertainty set $\mc{M}\subset \reals^{S\times A} \times \Delta^{SA}_S$ is $(s,a)$-rectangular if
\begin{equation}\label{eqn:sa_rectangular}
\mc{M} =  \bigtimes_{(s,a) \in [S]\times [A]} \mc{M}_{sa}, \ \mc{M}_{sa} \subset \reals \times \Delta_S, \ \forall (s,a)\in [S]\times[A].
\end{equation}
\end{defn}
Intuitively, $(s,a)$-rectangularity implies that the MDP parameter uncertainty is \emph{decoupled} between each state-action. A more general condition is if the parameter uncertainty is decoupled between different states but not between different actions within the same state.  
\begin{defn}[$s$-rectangular sets]\label{def:s_rectangular}
The uncertainty set $\mc{M}\subset \reals^{S\times A} \times \Delta^{SA}_S$ is $s$-rectangular if
\begin{equation}\label{eqn:s_rectangular}
\mc{M} = \bigtimes_{s \in [S]} \mc{M}_{s}, \ \mc{M}_{s} \subset \reals^A \times \Delta^A_S, \ \forall s \in [S].
\end{equation}
\end{defn}
$s$-rectangularity generalizes $(s,a)$-rectangularity---i.e. $(s,a)$-rectangularity implies $s$-rectangularity.
\begin{exmp}[Wind uncertainty]\label{ex:wind_uncertainty_2} \change{Consider the navigation problem presented in Example~\ref{ex:wind_uncertainty}.} 
If the wind pattern strictly switches between the discrete wind trends, then the transition uncertainty at state $s\in[S]$ is $\mc{P}_s = \{P^1_s, \ldots, P^N_s\}$. If the wind pattern is a mixture of the discrete wind trends, the transition uncertainty at state $s\in [S]$ is $\mc{P}_s = \{\sum_{i}\alpha_iP^i_{s} \ | \ \alpha \in \Delta_N\}$. Both wind patterns lead to $s$-rectangular uncertainty, given by $\mc{P} = \bigtimes_{s \in [S]} \mc{P}_s$.
\end{exmp}
We show that the rectangularity conditions indeed are sufficient for satisfying Assumption~\ref{assum:containment_condition} with respect to $f$~\eqref{eqn:bellman_operator} and $g^\pi$~\eqref{eqn:policy_operator}.
\begin{prop}
If $\mc{M}$ is compact and $s$-rectangular (Definition~\ref{def:s_rectangular}), $\mc{M}$ satisfies Assumption~\ref{assum:containment_condition} with respect to $f$~\eqref{eqn:bellman_operator} and $g^\pi$~\eqref{eqn:policy_operator} for all $\pi \in \Pi$.
\end{prop}
\begin{pf} We first show that $\mc{M}$ satisfies Assumption~\ref{assum:containment_condition} with respect to the Bellman operator. Given $s \in [S]$, $f_s(V, C, P)$ only depends on the $s$ component of $C$ and $P$. From Lemma~\ref{lem:operator_continuity}, $f_s$ is continuous in $(c_s, P_s)$. Let $(c_s^\star, P_s^\star)$ be the solution to $\argmin_{(c_s, P_s) \in \mc{M}_s} f_s(V, C, P)$ for all $\forall \ s\in [S]$. If $\mc{M}_s$ is compact, $(c_s^\star, P^\star_s) \in \mc{M}_s$. We can construct $C^\star  = [c_1^\star, \ldots, c^\star_S]$ and $P^\star = [P_1^\star,\ldots, P_S^\star]$. If $\mc{M}$ is $s$-rectangular, then $(C^\star, P^\star) \in \mc{M}$ and $(C^\star, P^\star) \in \argmin_{\change{(C,P)}\in\mc{M}} f_s(V, C, P) $ for all $s \in [S]$. We conclude that $\mc{M}$ satisfies Assumption~\ref{assum:containment_condition}.

Given $\pi \in \Pi$ and $s \in [S]$, $g^\pi_s$ only depends on $c_s$ and $P_s$ as well. We can similarly show that there exists an optimal parameter $(C^\star, P^\star) \in \argmin_{(C, P) \in \mc{M}} g^\pi_s(V, C, P)$ for all $s \in [S]$ such that $(C^\star, P^\star) \in \mc{M}$. \qed
\end{pf}
Beyond $s$-rectangularity, there are sets that satisfy Assumption~\ref{assum:containment_condition} with respect to specific value operators.
\begin{exmp}[Beyond rectangularity]
\begin{figure}
    \centering
    \includegraphics[width=0.6\columnwidth]{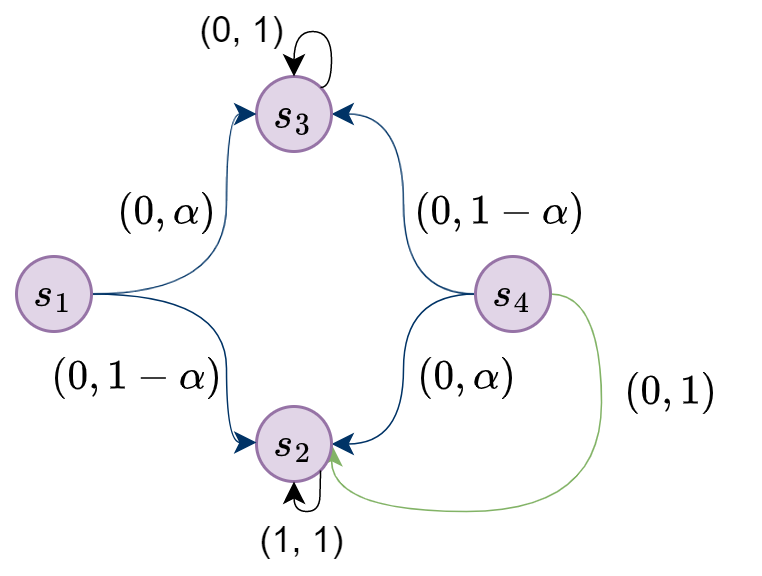}
    \caption{MDP with parameter coupling in transition probability across different states.}
    \label{fig:containment_general}
\end{figure}
In Figure~\ref{fig:containment_general}, we visualize a four state MDP with transition uncertainty $\mc{M}$ parameterized by $\alpha$. MDP states are the nodes and MDP actions are the arrows. Actions that transition to multiple states are visualized by multi-headed arrows. Each head has an associated tuple $(c_{sa}, p_{sa,s'})$ denoting its state-action cost and transition probability. All states have a single action except for state $s_4$, where two actions exist and are distinguished by different colors. Both $s_2$ and $s_3$ are absorbing states with a unique action, such that $V_2 = \frac{1}{1 - \gamma}$ and $V_3 = 0$ for both $f$ and $g^\pi$ for all $\pi \in \Pi$, where $\gamma$ is the discount factor. 

The states $s_1$ and $s_4$ have transition uncertainty parametrized by $\alpha \in [0, 1]$. Therefore, $\mc{M}$ violates $s$-rectangularity (Definition~\ref{def:s_rectangular}). The optimal cost-to-go values $V_1$ and $V_4$ occur at different $\alpha$'s. Therefore, $\mc{M}$ violates Assumption~\ref{assum:containment_condition} with respect to $f$. However, suppose that at $s_4$,  we only consider policies that exclusively choose the action  colored green in Fig.~\ref{fig:containment_general}. Then the expected cost-to-go at $s_4$, $V_4$, is independent of $\alpha$. The minimum and maximum values of $V_1$ under $\pi$ occur at $\alpha = 1$ and $\alpha = 0$, respectively. Therefore, $\mc{M}$ satisfies Assumption~\ref{assum:containment_condition} with respect to operator $g^\pi$ for all $\pi = [\pi_{s_1},\ldots, \pi_{s_4}]$ where $\pi_{s_4} =[1,0]$.
\end{exmp}
When Assumption~\ref{assum:containment_condition} is satisfied, the fixed point of $H$~\eqref{eqn:general_set_operator} contains its own supremum and infimum. 
\begin{thm}\label{thm:bound_containment}
If $h$~\eqref{eqn:value_operator} on $\reals^S\times \mc{M}$ satisfies Assumption~\ref{assum:containment_condition}, then there exists $\underline{m}, \overline{m} \in \mc{M}$ such that $\underline{h}$ and $\underline{h}$~\eqref{eqn:infsup_parameter_contractions} and their fixed points $\underline{X}$ and $\overline{X}$~\eqref{eqn:infsup_fixed_points} satisfies
\begin{equation}\label{eqn:containement_optimal_parameter_existence}
    \underline{h}(\underline{X}) = h(\underline{X}, \underline{m}) = \underline{X}, \ \overline{h}(\overline{X}) = h(\overline{X}, \overline{m}) = \overline{X}. 
\end{equation}
Additionally, $\underline{X}$ and $\overline{X}$ are the least and the greatest elements of $H$'s fixed point set $\mc{V}^\star$, $\underline{V}^\star, \overline{V}^\star$~\eqref{eqn:sup_inf_element_def} respectively, and both belong to $\mc{V}^\star$~\eqref{eqn:fixed_point_set}.
\[\textstyle\underline{X} = \underline{V}^\star, \ \overline{X} = \overline{V}^\star, \ \underline{X}, \overline{X} \in \mc{V}^\star.\]
\end{thm}
\begin{pf} 
From Theorem~\ref{thm:bound_containment_inexact}, $\underline{X}$ and $\overline{X}$ are the lower and upper bounds on the fixed point set $\mc{V}^\star$. We show that these are the infimum and supremum elements of $\mc{V}^\star$ by showing that they are also elements of $\mc{V}^\star$.
From Assumption~\ref{assum:containment_condition}, there exists $\underline{m}, \overline{m} \in \mc{M}$ such that ${h}_s(\underline{X}, \underline{m}) = \min_{m \in \mc{M}} h_s(\underline{X}, \underline{m})$ and ${h}_s(\overline{X}, \overline{m}) = \min_{m \in \mc{M}} h_s(\overline{X}, \overline{m})$ for all $s \in [S]$. Since $\underline{X}$ and $\overline{X}$ are fixed points of $h(\cdot, \underline{m})$ and $h(\cdot, \overline{m})$, we apply Corollary~\ref{cor:fixed_point_containment} to conclude that $\underline{X},\overline{X} \in \mc{V}^\star$.   
\qed \end{pf}
Our next result proves the relationship between $\underline{V}^B, \underline{V}^o, \underline{V}^r,$ $\overline{V}^B, \overline{V}^o, \overline{V}^r$ when $f, g^o$, and $g^r$ on $\reals^{S}\times \mc{M}$ satisfy Assumption~\ref{assum:containment_condition}.
\begin{thm} \label{thm:robust_control_set_bounds}
If $f, g^o, g^r$ satisfy Assumption~\ref{assum:containment_condition} on $\reals^{S}\times \mc{M}$, then the bounding elements~\eqref{eqn:bellman_bound_defs}~\eqref{eqn:optimistic_bound_defs}~\eqref{eqn:robust_bound_defs} of the corresponding fixed point sets $\mc{V}^B$,$\mc{V}^o$~\eqref{eqn:optimistic_fixed_point_set} and $\mc{V}^r$~\eqref{eqn:robust_fixed_point_set} are ordered as
\begin{equation}\label{eqn:robust_bounds_ordered}
    \underline{V}^B=\underline{V}^o\leq \underline{V}^r,\  \overline{V}^B=\overline{V}^r\leq \overline{V}^o.
\end{equation}
\end{thm}
\begin{pf}
Since $\underline{V}^o$ is the infimum element for the fixed point set $\mc{V}^o$~\eqref{eqn:optimistic_bound_defs}, we  can apply Theorem~\ref{thm:bound_containment} to derive
\begin{equation}\label{eqn:thm4_proof_0}
\textstyle\underline{V}^o = \min_{(C, P) \in \mc{M}} g^{o}(\underline{V}^o, C, P). 
\end{equation} 
By definition of \change{$\pi^o$}~\eqref{eqn:optimistic_policy}, $\min_{(C, P) \in \mc{M}} g^{o}(\underline{V}^o, C,P)= \min_{(C, P) \in \mc{M}} \min_{\pi\in \Pi}g^{\pi}(\underline{V}^o, C,P)$. As the two minima commute, 
\begin{equation}\label{eqn:thm4_proof_1}
 \min_{(C, P) \in \mc{M}} g^{o}(\underline{V}^o, C,P)=\min_{(C, P) \in \mc{M}} \min_{\pi\in \Pi} g^{\pi}(\underline{V}^o, C,P).   
\end{equation}
Combining~\eqref{eqn:thm4_proof_0} and~\eqref{eqn:thm4_proof_1}, $\underline{V}^o$ is exactly the unique fixed point of  $\min_{(C, P) \in \mc{M}} \min_{\pi\in \Pi} g^{\pi}(\cdot, C,P)$.  However, by applying Theorem~\ref{thm:bound_containment} to $f$ on $\reals^S\times\mc{M}$, $\underline{V}^B$ is also the unique fixed point of $\min_{(C, P) \in \mc{M}} \min_{\pi\in \Pi} g^\pi(\cdot, C,P)$. Therefore $\underline{V}^o = \underline{V}^B$. 

From~\eqref{eqn:robust_bound_defs}, $\underline{V}^r=\min_{(C, P)\in \mc{M}} g^{r}(\underline{V}^r, C,P)$, we can minimize over the policy space to lower bound $\underline{V}^r$ as 
\begin{equation}\label{eqn:thm4_proof_2}
\underline{V}^r\geq  \min_{\pi\in\Pi} \min_{(C, P)\in \mc{M}}g^{r}(\underline{V}^r, C, P).
\end{equation}
Since the right hand side of~\eqref{eqn:thm4_proof_2} is equivalent to $\underline{f}(\underline{V}^r)$,~\eqref{eqn:thm4_proof_2} is equivalent to $\underline{V}^r \geq \underline{f}(\underline{V}^r)$.
From Lemma~\ref{lem:parameterized_order_preserving}, $\underline{f}$ is order-preserving in $V$, we conclude that $\underline{V}^o = \underline{V}^\star \leq \underline{V}^r$.

From Theorem~\ref{thm:bound_containment}, $\overline{V}^r$ is the fixed point of \change{$\overline{g}^r$}, such that  
\begin{equation}\label{eqn:thm4_proof_3}
    \overline{V}^r = \max_{(C, P)\in \mc{M}} g^{r}(\overline{V}^r,  C, P).
\end{equation}
We apply $\min_{\pi}$ to both sides of~\eqref{eqn:thm4_proof_3} and use the definition of \change{$\pi^r$} to derive that $\overline{V}^r$ is the fixed point of $\min_{\pi\in\Pi} \max_{(C,P)\in\mc{M}} g^{\pi}(V^r,  C, P)$. From Assumption~\ref{assum:containment_condition}, there exists $(\overline{C}, \overline{P}) \in \mc{M}$ that maximizes $g^{\pi}(\overline{V}, C, P)$, so $\overline{V}^r$ equivalently satisfies
\[\overline{V}^r = \min_{\pi\in\Pi}g^\pi(\overline{V}^r,\overline{C}, \overline{P}). \]
From Corollary~\ref{cor:fixed_point_containment}, this implies that $\overline{V}^r \in \mc{V}^B$ and therefore $\overline{V}^r \leq \overline{V}^B$.
Next we show $\overline{V}^B \leq \overline{V}^r$. From Theorem~\ref{thm:bound_containment}, $\overline{V}^B$ is the fixed point of $\overline{f}$, such that $\overline{V}^B = \max_{(C,P)\in\mc{M}} \min_\pi g^{\pi}(\overline{V}^B, C, P)$,
From  the min-max inequality, 
\[\overline{V}^B \leq  \min_{\pi\in\Pi} \max_{(C,P)\in\mc{M}} g^{\pi}(\overline{V}^B, C, P).\] 
Since \change{$\pi^r \in \Pi$}, 
\begin{equation}\label{eqn:thm4_proof_4}
   \overline{V}^B \leq \max_{(C,P) \in \mc{M}} g^{r}(\overline{V}^B, C, P). 
\end{equation}
The right-hand side of~\eqref{eqn:thm4_proof_4} is  $\overline{g}^{r}(\overline{V}^B)$~\eqref{eqn:infsup_parameter_contractions}, such that~\eqref{eqn:thm4_proof_4} is equivalent to $\overline{V}^B \leq \overline{g}^{r}(\overline{V}^B)$. We consider the sequence $V^{k+1} = \overline{g}^r(V^k)$ where $V^1 = \overline{V}^B$. Since $\overline{g}^r$ is a contraction, $\lim_{k\to \infty}V^k = V^r$, the fixed point of $\overline{g}^r$. From Lemma~\ref{lem:parameterized_order_preserving}, $\overline{g}^r$ is order preserving. Therefore $\overline{V}^B = V^1 \leq V^r$. 

Finally, Theorem~\ref{thm:bound_containment} implies that $\overline{V}^o$ is the fixed point of $\overline{g}^{o}$: $\overline{V}^o = \max_{(C,P)\in\mc{M}} g^{o}(\overline{V}^o, C, P)$.
By construction, $\overline{V}^o \geq \min_{\pi\in\Pi}\max_{(C,P)\in\mc{M}} g^\pi(\overline{V}^o, C, P)$.
From the min-max inequality, \[\min_{\pi\in\Pi}\max_{(C,P)\in\mc{M}} g^\pi(\overline{V}^o, C, P)\geq\max_{(C,P)\in\mc{M}}\min_{\pi\in\Pi} g^\pi(\overline{V}^o,  C, P),\]
such that the right hand side of the inequality is equivalent to $\overline{f}(\overline{V}^o)$. Following the monotonicity properties of the Bellman operator $f$~\cite[Thm.6.2.2]{puterman2014markov}, we conclude that $\overline{V}^o \geq \overline{V}^B$.  
\qed \end{pf}
\begin{rem}\label{rem:robust_low_variance} Through our fixed-point analysis, we see that in addition to having the best worst-case performance among $\{\mc{V}^o, \mc{V}^B, \mc{V}^r\}$,  $\mc{V}^r$ also has the smallest variation in performance for the same uncertainty set $\mc{M}$.
\end{rem}
Finally, we generalize the $s$-rectangularity condition by showing that the optimistic and robust policies exist  when the MDP parameter set $\mc{M}$ satisfies Assumption~\ref{assum:containment_condition}. 
\begin{cor}[Robust MDP under Assumption~\ref{assum:containment_condition}]
If $\mc{M}$ is compact and convex, and $f, g^o, g^r$ satisfy Assumption~\ref{assum:containment_condition} on $\reals^S\times \mc{M}$, then $W^o$~\eqref{eqn:optimistic_value_function} and $W^r$~\eqref{eqn:robust_value_function} are the infimum and supremum value vectors for the policy evaluation operator under $\pi^o$~\eqref{eqn:optimistic_policy} and $\pi^r$~\eqref{eqn:robust_policy}, respectively. 
\begin{equation}
    W^o_s = \inf_{V \in \mc{V}^o} \change{V_s}, W^r_s = \sup_{V\in\mc{V}^r} \change{V_s}, \ \forall s \in [S],
\end{equation}
where $\mc{V}^o$~\eqref{eqn:optimistic_fixed_point_set} and $\mc{V}^r$~\eqref{eqn:robust_fixed_point_set} are the fixed point sets of policies $\pi^o$ and $\pi^r$ under parameter uncertainty $\mc{M}$, respectively. 
\end{cor}
\begin{pf}
When $f$ satisfies Assumption~\ref{assum:containment_condition} on $\reals^{S} \times \mc{M}$, Theorem~\ref{thm:bound_containment} shows that $\underline{V}^B = W^o, \ \overline{V}^B = W^r.$
If $f, g^o,$ and $g^r$ also satisfies Assumption~\ref{assum:containment_condition} on $\reals^S\times \mc{M}$, then we apply Theorem~\ref{thm:robust_control_set_bounds} to derive $W^o = \underline{V}^o$ and $W^r = \overline{V}^r$. This proves the corollary statement. 
\qed \end{pf}
\begin{rem}
When Assumption~\ref{assum:containment_condition} is not satisfied, $W^o$ and $W^r$ still bound $\underline{V}^o$ and $\overline{V}^r$. This result is also stated in~\cite{wiesemann2013robust}. 
\end{rem}

\begin{figure}
    \centering
    \includegraphics[width=\columnwidth]{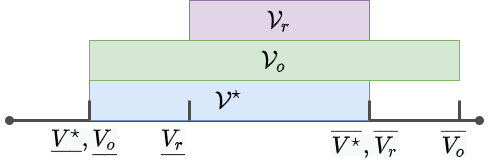}
    \caption{Illustration of Theorem~\ref{thm:robust_control_set_bounds}. The purple, green, blue regions indicate the ranges of $\mc{V}^r$, $\mc{V}^o$, and $\mc{V}^B$, respectively.}
    \label{fig:set_inclusion}
\end{figure}

\section{Value iteration for fixed point set computation}\label{sec:set_VI_propagation}
In the previous sections, we proved the existence of a fixed point set for value operators with compact parameter uncertainty sets and re-interpreted robust control through our techniques.  Next, we derive an iterative algorithm for computing the bounds of the fixed point set $\mc{V}$ given a value operator $h$ and parameter uncertainty set $\mc{M}$.


\textbf{Algorithm Sketch}. Based on the set-based value iteration~\eqref{eqn:set_based_value_iteration}, we iteratively find the one-step bounds of $H(\mc{V}^k)$ to converge the bounds of the fixed point set.

For any compact set $\mc{V} \in \mc{K}(\reals^S)$, the one step bounds of $H(\mc{V})$ are equivalent to the one-step output of the bound operators $\underline{h}$ and $\overline{h}$~\eqref{eqn:infsup_parameter_contractions} applied to the extremal points of $\mc{V}$. 
\begin{thm}[One step $H$ bounds]\label{thm:one_step_bounds}
Consider a set operator $H$~\eqref{eqn:general_set_operator} and its bound operators $\underline{h}$ and $\overline{h}$~\eqref{eqn:infsup_parameter_contractions} induced by $h$ on  $\reals^S\times\mc{M}$~\eqref{eqn:value_operator}. For a compact set $\mc{V}\subset \reals^S$, $H(\mc{V})$ is bounded by $\underline{h}(\underline{V})$ and $\overline{h}(\overline{V})$~\eqref{eqn:infsup_parameter_contractions} as
\begin{equation}\label{eqn:one_step_bounds}
    \underline{h}(\underline{V}) \leq V \leq \overline{h}(\overline{V}), \quad \forall \ V\ \in H(\mc{V}).
\end{equation}
where $\underline{V}$ and $\overline{V}$~\eqref{eqn:sup_inf_element_def} are the extremal elements of $\mc{V}$. If $h$ satisfies Assumption~\ref{assum:containment_condition} on $\reals^S\times\mc{M}$ and $\underline{V}, \overline{V} \in \mc{V}$, then $\underline{h}(\underline{V})$ and $\overline{h}(\overline{V})$ are the supremum and infimum elements of $H(\mc{V})$, respectively--- for all $s\in[S]$, $\underline{h}_s(\underline{V})$ and $\overline{h}_s(\overline{V})$ satisfy
\begin{equation}\label{eqn:one_step_bounds_exact}
  \underline{h}_s(\underline{V}) =  \underset{(V, m) \in \mc{V}\times\mc{M}}{\inf}{h_s(V, m)}, \ \overline{h}_s(\overline{V}) =   \underset{(V, m) \in \mc{V}\times\mc{M}}{\sup}h_s(V, m). 
\end{equation}
\end{thm}
\begin{pf}
For all $s \in [S]$, $h_s(V, m) \leq \overline{h}_s(V)$ for all $m \in \mc{M}$. If $h$ is $K(V)$-Lipschitz and $\alpha$-contractions in $\mc{M}$, then $\overline{h}$ is order-preserving (Lemma~\ref{lem:parameterized_order_preserving}) such that $ \overline{h}_s(V) \leq  \overline{h}_s(\overline{V})$ for all $V \in \mc{V}$. We conclude that
\begin{equation}\label{eqn:thm6_proof_0}
    h(V, m) \leq \overline{h}(\overline{V}), \ \forall (V, m)\in\mc{V}\times\mc{M}.
\end{equation}
Since $\overline{h}$ is an upper bound, and $\sup$ is the least upper bound, it holds  that $ \sup_{V, m}[h(V, m)]_s \leq \overline{h}(\overline{V})$. We use the definition of ${H(\mc{V})}$~\eqref{eqn:general_set_operator} to conclude that $V \leq \overline{h}(\overline{V})$ for all $V \in H(\mc{V})$. The inequality $\underline{h}(\underline{V}) \leq V \ \forall V \in H(\mc{V})$ can be similarly proved. 

If $h$ satisfies Assumption~\ref{assum:containment_condition} on $\reals^S\times \mc{M}$ and $\underline{V}, \overline{V} \in \mc{V}$, Assumption~\ref{assum:containment_condition} states that there exists $\underline{m}\in \mc{M}$ such that $h(\underline{V}, \underline{m}) = \underline{h}(\underline{V})$. Therefore, $\underline{h}(\underline{V}) \in  H(\mc{V})$. Since $\underline{h}(\underline{V})$ also lower bounds all the elements of $H(\mc{V})$, it is the infimum element of $H(\mc{V})$. The fact that the greatest element of $H(\mc{V})$ is $\overline{h}(\overline{V})$ can be similarly proved.
\qed \end{pf}
Based on Theorem~\ref{thm:one_step_bounds}, we propose the following bound approximation algorithm of the fixed point set $\mc{V}^\star$~\eqref{eqn:fixed_point_set} for a set-valued operator $H$~\eqref{eqn:value_operator}. 
\begin{algorithm}
\caption{Bound approximation of the fixed point set $\mc{V}$}
\begin{algorithmic}[1]
\Require \(\mc{C}\), \(\mc{P}\), \({V}^0\), $\epsilon$.
\Ensure \(\underline{V} \), \(\overline{V}\)
\State{\(\underline{V}^0 := \overline{V}^0 := V^0\)}
\State{\(e^0 = \frac{1 - \gamma}{\gamma}\epsilon\)}
\While{\(\frac{\gamma}{1 - \gamma}e^{k} \geq \epsilon \)}
    \State{\(\underline{V}^{k+1}_s = \min_{m\in \mc{M}} h_s(\underline{V}^k, m), \quad \forall s \in [S]\)}\label{algline:min_arg_m}
    \State{\(\overline{V}^{k+1}_s = \max_{m\in \mc{M}} h_s(\overline{V}^k, m), \quad \forall s \in [S]\)}\label{algline:max_arg_m}
    \State{\( e^{k+1}= \max \Big\{\norm{\underline{V}^{k+1} - \underline{V}^k}, \norm{\overline{V}^{k+1} - \overline{V}^k}\Big\}\)}
    \State{\(k = k+1\)}
\EndWhile
\end{algorithmic}
\label{alg:set_VI}
\end{algorithm}

\subsection{Computing one-step optimal parameters}
Algorithm~\ref{alg:set_VI} is stated for a general MDP parameter set $\mc{M}$ and does not specify how to compute lines~\ref{algline:min_arg_m} and~\ref{algline:max_arg_m}.  Here we discuss solution methods for different shapes of $\mc{M}$.
\begin{enumerate}
    \item \textbf{Finite $\mc{M}$}. If $\mc{M} = \{m_1, \ldots, m_N\}$ is a set with finite number of elements, we can directly compute line~\ref{algline:min_arg_m} as 
    \begin{equation}\label{eqn:finite_M_computation}
       \underline{V}^{k+1} = \min\Big\{ h_s(\underline{V}^k, m_i) \ | \ i = \{1,\ldots, N\}\Big\}. 
    \end{equation}
    For line~\ref{algline:max_arg_m}, we replace $\min$ with $\max$ in~\eqref{eqn:finite_M_computation}.
    \item \textbf{Convex $\mc{M}$}. When $\mc{M}$ is a convex set, the computation depends on $h$. If $h = g^\pi$ is the policy operator, lines~\ref{algline:min_arg_m} and~\ref{algline:max_arg_m} can be solved as convex optimization problems. 
    If $h$ is the Bellman operator $f$, lines~\ref{algline:min_arg_m} and~\ref{algline:max_arg_m} take on min-max formulation and is NP-hard to solve in the general form~\cite{wiesemann2013robust}. When $\mc{M}$ can be characterized by an ellipsoidal set of parameters, the solutions to lines~\ref{algline:min_arg_m} and~\ref{algline:max_arg_m} is given in~\cite{wiesemann2013robust}. 
\end{enumerate}
We recall the stochastic path planning problem from Example~\ref{ex:wind_uncertainty} with the two different parameter uncertainty scenarios. When the wind field uncertainty is discrete, $\mc{M}$ is finite, when wind field is a combination of the major wind trends, $\mc{M}$ is convex. 

\subsection{Algorithm Convergence Rate}
When lines~\ref{algline:min_arg_m} and~\ref{algline:max_arg_m} are solvable, Algorithm~\ref{alg:set_VI} asymptotically converges to \emph{approximations} of the bounding elements of $\mc{V}^\star$. If $\mc{M}$ satisfies Assumption~\ref{assum:containment_condition} with respect to $h$, Algorithm~\ref{alg:set_VI} derives the exact bounds of $\mc{V}$. 
Algorithm~\ref{alg:set_VI} has similar rates of convergence in Hausdorff distance as standard value iteration using $h$ on $\reals^S$.
\begin{thm}\label{thm:vi_convergence}
Consider the value operator $h$, compact uncertainty set $\mc{M}$, and the fixed point set $\mc{V}^\star$ of the set-based operator $H$~\eqref{eqn:general_set_operator} induced by $h$ on $\reals^{S}\times \mc{M}$. If $\mc{M}$ satisfies Assumption~\ref{assum:containment_condition} with respect to $h$, then at each iteration $k$,
\begin{equation}\label{eqn:thm6_linear_convergence}
 \textstyle\norm{\underline{V}^{k+1} - \underline{V}^\star}_{\change{\infty}} \leq \alpha\norm{\underline{V}^k - \underline{V}^\star}_{\change{\infty}}, \norm{\overline{V}^{k+1} - \overline{V}}_{\change{\infty}} \leq \alpha\norm{\overline{V}^k - \overline{V}^\star}_{\change{\infty}},   
\end{equation}
where all norms are infinity norms, and $\underline{V}^\star, \overline{V}^\star$ are the infimum and supremum bounds of $\mc{V}$, respectively. 
At Algorithm~\ref{alg:set_VI}'s termination, $\underline{V}^k, \overline{V}^k$ satisfies
\begin{equation}
    \max\{\norm{\underline{V}^k - \underline{V}^\star}_{\change{\infty}}, \norm{\overline{V}^k - \overline{V}^\star}_{\change{\infty}}\} < \epsilon.
\end{equation}
\end{thm}
\begin{pf}
From Algorithm~\ref{alg:set_VI}, $\overline{V}^{k+1} = \overline{h}(\overline{V}^k)$. From Lemma~\ref{lem:parametrized_contraction}, $\overline{h}$ is an $\alpha$-contraction. We obtain  $\norm{\overline{V}^{k+1} - \overline{V}^\star}_{\change{\infty}} \leq \alpha \norm{\overline{V}^k - \overline{V}^\star}_{\change{\infty}}$ and note that~\eqref{eqn:thm6_linear_convergence} holds by induction. 
Next, we apply triangle inequality to $\norm{\overline{V}^k - \overline{V}^\star}_{\change{\infty}}$ to derive
\begin{equation}\label{eqn:thm7_proof_0}
    \norm{\overline{V}^{k} - \overline{V}^\star}_{\change{\infty}} \leq \norm{\overline{V}^{k} - \overline{V}^{k+1}}_{\change{\infty}} + \norm{\overline{V}^{k+1} - \overline{V}^\star}_{\change{\infty}}.
\end{equation} 
We can then use $\norm{\overline{V}^{k+1} - \overline{V}^\star}_{\change{\infty}} \leq \alpha \norm{\overline{V}^k - \overline{V}^\star}_{\change{\infty}}$ to bound~\eqref{eqn:thm7_proof_0} as $\norm{\overline{V}^{k} - \overline{V}^\star}_{\change{\infty}} \leq \frac{1}{1- \alpha} \norm{\overline{V}^{k} - \overline{V}^{k+1}}_{\change{\infty}}$.
A similar argument can show that $\norm{\underline{V}^{k} - \underline{V}^\star}_{\change{\infty}} \leq \frac{1}{1- \alpha} \norm{\underline{V}^{k} - \underline{V}^{k+1}}_{\change{\infty}} $. 
When Algorithm~\ref{alg:set_VI}'s while condition is satisfied, $\max\Big\{\norm{\overline{V}^{k} - \overline{V}^\star}_{\change{\infty}}, \norm{\underline{V}^{k} - \underline{V}^\star}_{\change{\infty}}\Big\} \leq \epsilon.$ This concludes our proof.
\qed \end{pf}
In particular, the Bellman operator $f$ and policy operator $g^\pi$ are $\gamma$-contractive on $\reals^S$, where $\gamma$ is the discount factor, therefore\change{, } Theorem~\ref{thm:one_step_bounds} applies with $\alpha = \gamma$. 
\begin{rem}
Theorem~\ref{thm:vi_convergence} implies that at the termination of Algorithm~\ref{alg:set_VI}, the fixed point set $\mc{V}^\star$ can be \emph{over-approximated} by
\[\mc{V}^\star \subseteq \mc{V}_{approx} := \prod_{s \in[S]}[\underline{V}_s^{k+1} - \epsilon, \overline{V}_s^{k+1} + \epsilon] , \]
where $k$ is the last iterate before Algorithm~\ref{alg:set_VI} terminates.
\end{rem}
\section{Path Planning in Time-varying Wind Fields}\label{sec:example}
We apply set-based value iteration to wind-assisted probabilistic path planning of a balloon in strong, uncertain wind fields~\cite{wolf2010probabilistic}. MDP as a model for wind-assisted path planning of balloons in the stratosphere and exoplanets has recently gained traction~\cite{wolf2010probabilistic,bellemare2020autonomous}. Discrete state-action MDPs have been shown to be a viable high-level path planning model~\cite{wolf2010probabilistic} for such applications. 

\textbf{Mission Objective}. In the two dimensional wind-field, we assume that the wind-assisted balloon is tasked with reaching target state $(8,8)$ in Figure~\ref{fig:wind_field} using minimum fuel. 

\textbf{Uncertain Wind Fields}. By collecting a set of wind data on the environment's wind field, an MDP can be created and a policy that handles stochastic planning can be deployed. However, wind can be a time-varying factor that causes the \emph{expected} optimal policy to have \emph{worse-than-expected} worst-case performance. We built an ideal uncertain wind field to demonstrate how the set Bellman operator can be used to predict the best and worst-case behavior of a robust policy. 

\textbf{MDP Modeling Assumptions}. Following the framework described in~\cite{wolf2010probabilistic}, we model the path planning problem in an uncertain wind field as an infinite horizon, discounted MDP with discrete state-actions in a two-dimensional space. While balloons typically traverse in three dimensions, we assume that the wind is consistent in the vertical direction and that the final target is any vertical position along the given two-dimensional coordinates. As a result, we can disregard the vertical position during planning. 
\begin{figure}[ht!]
    \centering
    \subfloat[\centering Wind field traversed by the balloon, discretized into $81$ states. ]{{ \includegraphics[width=0.4\columnwidth]{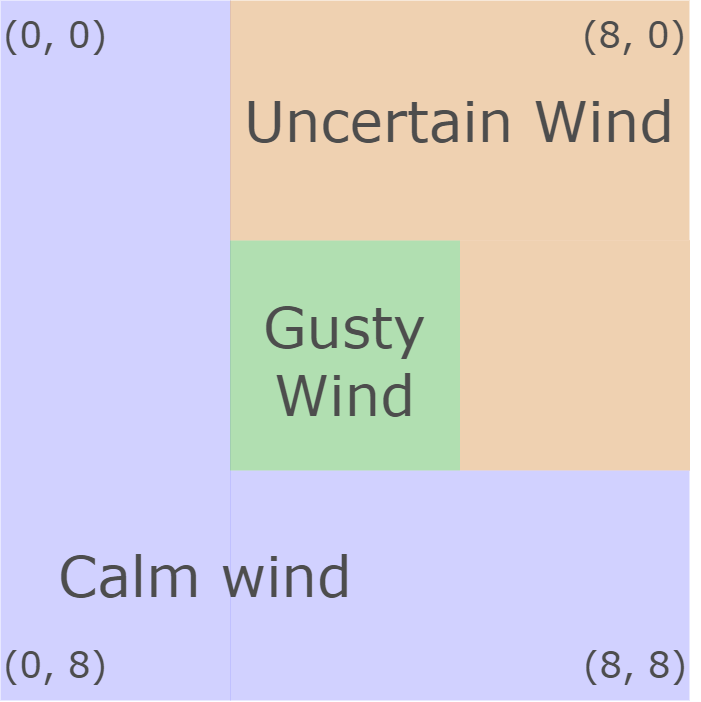}}}%
    \quad
    \subfloat[\centering At each state, $9$ actions corresponding to different thrust vectors are available.]{{\includegraphics[width=0.4\columnwidth]{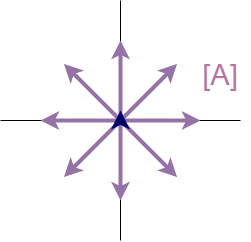}}}%
    \caption{}%
    \label{fig:wind_field}%
\end{figure}

\textbf{States}. A total of $81$ states represent the two-dimensional space, composed of three different regions characterized by their wind variability as shown in Figure~\ref{fig:wind_field}.
\begin{enumerate}
        \item \textbf{Calm wind}. In calm states $S_{calm}$, the wind magnitude varies uniformly between $[0, 0.5]$, and the wind direction is uniformly sampled between $[0, 2\pi]$.
        $S_{calm} = \{(i, j) \ | \ (0, 0)\leq(i, j) \leq (2, 8), \ (6, 0) \leq (i, j) \leq (8, 8)\}.$
        \item \textbf{Gusty wind}. In states with gusts $S_{gusty}$, wind magnitude is consistently $1$, while the wind direction is uniformly sampled between $[0, 2\pi]$. $S_{gusty} = \{(i,j) \ | \ (3, 3) \leq (i, j) < (6,6) \}$.
        \item \textbf{Unreliable wind}. In unreliable states $S_{unreliable}$, a predictable wind front occasionally moves across an otherwise windless region. In other words, the wind magnitude is either $0$ or $1$ and the wind direction varies uniformly between $[\pi/4, \pi/2]$.
    \end{enumerate}
    
\textbf{Actions}. The balloon is equipped with an actuator that provides a constant thrust of $1$ in $8$ discretized directions shown in Figure~\ref{fig:wind_field}b. The only stationary action vector with zero magnitude is highlighted in blue in the center of Figure~\ref{fig:wind_field}b. We assume that the actuation force is enough to move the balloon across one state in wind with magnitude $\leq 0.5$, and is otherwise not strong enough to overcome wind effects. 
\begin{figure}[ht!]
    \centering
    \subfloat[\centering State transition in calm wind. ]{{\includegraphics[width=0.27\columnwidth]{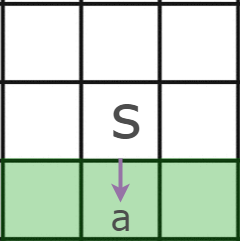}}}%
    \quad
    \subfloat[\centering State transition in unreliable wind.]{{\includegraphics[width=0.27\columnwidth]{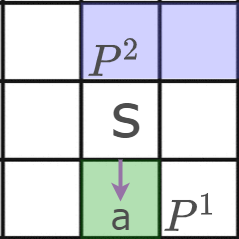}}}%
    \quad
    \subfloat[\centering State transition in gusty wind.]{{\includegraphics[width=0.27\columnwidth]{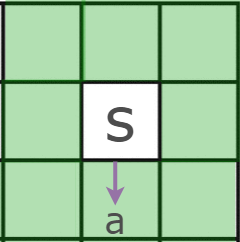}}}%
    \caption{Transition probabilities for the three different wind regions.}%
    \label{fig:neighbors}%
\end{figure}

\textbf{Transition Probabilities}. \change{The transition probabilities are region-dependent. In the states $[S_{calm}]$ and $[S_{gusty}]$, the transition dynamics are stochastic but stationary in time. In the states $[S_{unreliable}]$, the transition dynamics are stochastic but change over time. We define}  the following neighboring states \change{for each state $s \in [S]$}. 
\begin{enumerate}
    \item $\mc{N}(s)$: all  $8$ neighboring states of state $s$. 
    \item $\mc{N}(s,a, 0)$: the neighboring state of $s$ in the direction of $a$.
    \item $\mc{N}(s,a, 1)$: the neighboring state of $s$ in the direction of $a$ plus the two adjacent states as shown in Figure~\ref{fig:neighbors}a.
    \item $\mc{N}(s,a,2)$: the up and upper-right neighbors of $s$, as shown in Figure~\ref{fig:neighbors}b. 
\end{enumerate}
In the calm wind region, the transition probabilities are given by
\begin{equation}
    P_{sa, s'} = \begin{cases} \frac{1}{\mc{N}(s,a, 1)}, & s' \in \mc{N}(s,a,1) \\
    0 & \text{otherwise}
    \end{cases}, \ \forall \  s \in [S_{calm}].
\end{equation}
In the gusty wind region, the transition probabilities are given by 
\begin{equation}
    P_{sa, s'} = \begin{cases} \frac{1}{\mc{N}(s)}, & s' \in \mc{N}(s) \\
    0 & \text{otherwise}
    \end{cases}, \ \forall \  s \in [S_{gusty}], \ \forall \ a \in [A].
\end{equation}
In the unreliable wind region, the transition probabilities vary between transition dynamics $P_s^1$ and $P_s^2$. 
\begin{equation}
    P^1_{sa, s'} = \begin{cases}1, & s' \in \mc{N}(s, a, 0) \\
    0 & \text{otherwise}
    \end{cases}, \ \forall \  s \in [S_{gusty}], \ \forall \ a \in [A].
\end{equation}
\begin{equation}
    P^2_{sa, s'} = \begin{cases}0.5, & s' \in \mc{N}(s, a, 2) \\
    0 & \text{otherwise}
    \end{cases}, \ \forall \  s \in [S_{gusty}], \ \forall \ a \in [A].
\end{equation}
Collectively, $P^1_s$ and $P^2_s$ collectively form the uncertainty set $\mc{P}_s\subset \Delta_{S}^A$ defined at each state.
\begin{equation}\label{eqn:variable_wind_uncertainty}
    \mc{P}_s = \{P^i_{sa} \ | \ i \in \{1,2\}, a \in [A]\}, \ \forall s \in [S_{unreliable}].
\end{equation} 
\textbf{Cost}. We define the following state-action cost to achieve the mission objective: at each state-action, the cost is the sum of the current distance from target position $s_{targ} = (8,8)$, as well as the fuel expended by given action. 
\[C((i,j), a) = \sqrt{(i - s_{targ}[0])^2 + (j- s_{targ}[1])^2} + \half\norm{a}_2.\]  
We take $a = 1$ for all actions except for the staying still action, where $a = 0$. 
\subsection{Bellman, optimistic policy, and robust policy}
We first compute the optimistic and robust bounds of the MDP with parameter uncertainty in $\mc{P}$ when $s \in [S_{unreliable}]$ by running Algorithm~\ref{alg:set_VI}. The results are shown in Figure~\ref{fig:set_bounds}.
\begin{figure}[ht!]
    \centering 
    \subfloat[\centering Optimistic case with expected objective of $54.2$]{{ \includegraphics[width= 0.48\columnwidth]{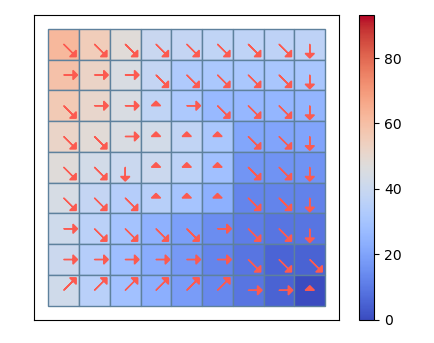}}}%
    \
    \subfloat[\centering Robust case with expected objective of $96.7$.]{{\includegraphics[width=0.48\columnwidth]{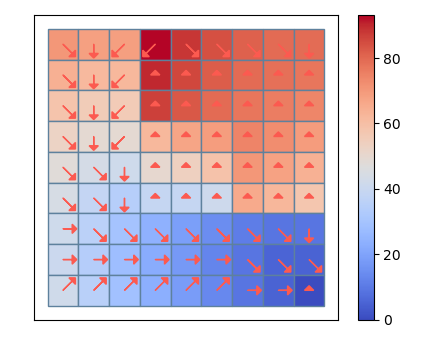}}}%
    \caption{}%
    \label{fig:set_bounds}%
\end{figure}

We denote the optimistic policy as $\pi^o$ and the robust policy as $\pi^r$, and derive the bounds of their respective value vector sets $\mc{V}^o$~\eqref{eqn:optimistic_fixed_point_set}  and $\mc{V}^r$~\eqref{eqn:robust_fixed_point_set} using Algorithm~\ref{alg:set_VI}. The output is compared against the bounds of the set-based Bellman operator's fixed point set $\mc{V}^\star$ in Table~\ref{tab:wind_bound_values}. 
\begin{table}[h!]
\begin{center}
\begin{tabular}{ccc}
\hline
Set & Maximum value & Minimum value   \\
\hline
$\mc{V}^\star $& 70.61 & 62.25\\
$\mc{V}^o $& 101.58 & 62.25 \\
$\mc{V}^r $& 70.63 & 70.52\\
\hline
\end{tabular}
\end{center}

\caption{Bellman, optimistic policy, robust policy value bounds of the uncertain wind field.}\label{tab:wind_bound_values}
\end{table}

\textbf{Time-varying wind field}
Next, we consider a time-varying wind field: at each time step $k$, the transition probability $P^k$ is chosen at random from $\mc{P}$~\eqref{eqn:variable_wind_uncertainty}. In this time-varying wind field, we compare three different policy deployments: 1) stationary optimistic policy $\pi^o$ as policy operator $g^o$~\eqref{eqn:optimistic_fixed_point_set}, 2) stationary robust policy $\pi^r$ as policy operator $g^r$~\eqref{eqn:robust_fixed_point_set}, and 3) dynamically changing policy that is optimal for the MDP $([S], [A], P^k, C, \gamma)$ as $f$~\eqref{eqn:bellman_operator}. These three different policy deployments are given by
\begin{align}
    V^{k+1} & = g^o(V^k, C, P^k), \label{eqn:time_varying_optimistic_pol}\\
    V^{k+1} & = g^r(V^k, C, P^k), \label{eqn:time_varying_robust_pol}\\
    V^{k+1} & = f(V^k, C, P^k) \label{eqn:time_varying_bellman_pol}.
\end{align}
The resulting cost-to-go at state $s_{orig} = [0,0]$ is plotted in Figure~\ref{fig:time_varying_wind_comparison}. Here, we see that the optimistic policy deployment~\eqref{eqn:time_varying_optimistic_pol} has the greatest variation in value over the course of $50$ MDP time steps. Both the robust policy deployment~\eqref{eqn:time_varying_robust_pol} and the dynamically changing policy deployment~\eqref{eqn:time_varying_bellman_pol} achieve better upper-bound at each MDP iteration. The dynamically changing policy deployment~\eqref{eqn:time_varying_bellman_pol} achieves less than $70$ in cost-to-go on average, which is the best among all three deployments. As we discussed in Remark~\ref{rem:robust_low_variance}, the robust policy deployment has the smallest variance in value in the presence of wind uncertainty, achieving a value difference of less than $0.1$. 
\begin{figure}[ht!]
    \centering 
    \subfloat[\centering Optimistic Policy with $\mc{V}^o$'s bounding values.]{{ \includegraphics[width= 0.8\columnwidth]{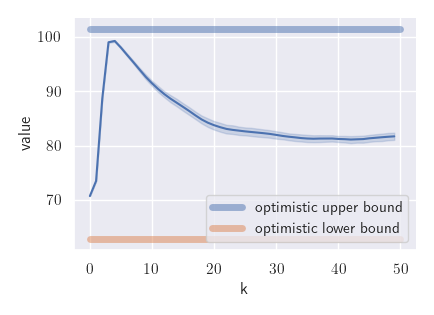}}}%
    \\
    \subfloat[\centering  Robust Policy with $\mc{V}^r$'s bounding values.]{{\includegraphics[width=0.8\columnwidth]{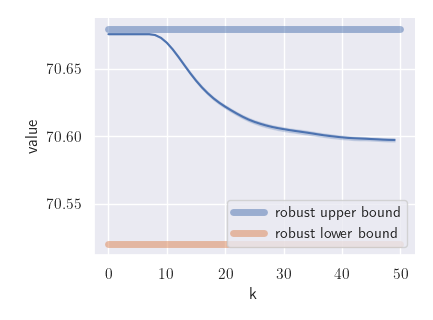}}}%
     \\
    \subfloat[\centering Dynamically changing policy with $\mc{V}^B$'s bounding values.]{{\includegraphics[width=0.8\columnwidth]{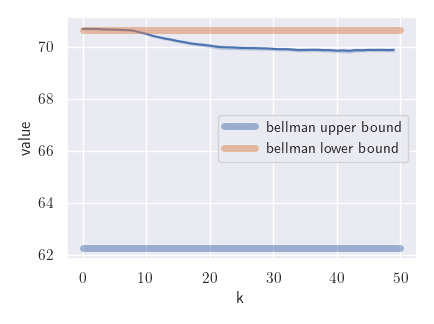}}}%
    \caption{Comparison of robust policy, optimistic policy, and Bellman policy's value trajectories in time-varying wind fields. Center blue line is the average over $50$ trials. The shaded blue region denotes the standard variation. The top and bottom lines are the supremum and infimum values of the fixed points.}%
    \label{fig:time_varying_wind_comparison}%
\end{figure}

\textbf{Sampled solutions}.
We can compute a sampled MDP model based on $50$ samples of wind vectors for each state. Based on these samples, we add the action vector and compute the statistical distribution of state transitions. We then compute the value of these \emph{stationary} sampled MDPs, and compare $9$ randomly selected states' values. The resulting scatter plot is shown in Figure~\ref{fig:wind_sampling}.
\begin{figure}
    \centering
    \includegraphics[width=1\columnwidth]{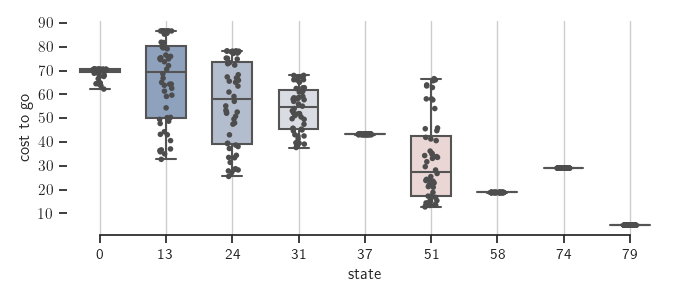}
    \caption{Comparison of different optimal value vectors under the Bellman operator for $50$ randomly sampled MDPs. On the x-axis, the state number is computed as $i\times 9 + j$.}
    \label{fig:wind_sampling}
\end{figure}
\section{Conclusion}
In this paper, we categorized a class of operators utilized to solve Markov decision processes as value operators and lifted their input space from vectors to \emph{compact sets} of vectors. We showed using fixed point analysis that the set extensions of value operators have fixed point sets that remain invariant given a compact set of MDP parameter uncertainties. These sets were applied to robust dynamic programming to further enrich existing results and generalize the $k$-rectangularity assumption for robust MDPs. Finally, we applied our results to a path planning problem for time-varying wind fields. For future work, we plan on applying set-based value operators to stochastic games in the presence of uncoordinated players such as humans, as well as applying value operators to reinforcement learning to synthesize robust learning algorithms. 
\bibliographystyle{plain}
\bibliography{reference}
\appendix
\section{Set sequence convergence}\label{app:1}
\begin{lem}
\label{basicdh}
Let $\{\mc{V}_n\}\subseteq \mc{K}(\reals^{S})$ be a converging sequence for $d_{\mc{K}}$ with $\mc{V}_n \to\mc{V} $ as $n\to\infty$. For all $V\in\mc{V}$, there exists a converging subsequence $\{V^{\varphi(n)}\}_{n \in \mathbb{N}}$ whose limit is $V$ for $\norm{\cdot}$. 
\end{lem}
\begin{pf}
Let $V\in \mc{V}$. We can define the strictly increasing function $\varphi$ on $\naturals$ as follows:
$\varphi(0):=0$ and for all $n\in \naturals$, $\varphi(n+1):=\min\{j>\varphi(n) \mid \exists\, V^j\in \mc{V}^j, \norm{V-V^j}=d(V,\mc{V}^j)\leq (n+1)^{-1}\}$. Finally, as for all $n\in\naturals^*$, $\norm{V-V^{\varphi(n)}}\leq (\varphi(n)+1)^{-1}$, the result holds. 
\qed \end{pf}
\section{Proof of Lemma~\ref{lem:operator_continuity}}\label{app:2}
\begin{pf}Let $(V,m)\in \reals^S\times \mc{M}$ and consider a sequence $\{(V_k,m_k)\}_{k\in\mathbb{N}} \subset \reals^{S}\times \mc{M}$ that converges to $(V,m)$. It holds that
$\norm{h(V_k,m_k)-h(V,m)} $ $\leq \norm{h(V_k,m_k)-h(V,m_k)} + \norm{h(V,m_k)-h(V,m)}$, where from the $\alpha$-contractive property of $h(\cdot, m^k)$, $\norm{h(V_k,m_k)-h(V,m_k)}  \leq \alpha\norm{V_k-V}$. From the $K(V)$-Lipschitz property of $h(V,\cdot)$, \[\textstyle\norm{h(V,m_k)-h(V,m)} \leq K(V)\norm{m_k-m}.\]
As both $\lim_{k\to\infty}\norm{V_k-V}\to0$ and $\lim_{k\to\infty}\norm{m_k-m}\to0$,  $\norm{h(V_k,m_k)-h(V,m)}\to 0$ and $h$ is continuous. \qed
\end{pf}
\section{Proof of Lemma~\ref{lem:parametrized_contraction}}\label{app:3}
\begin{pf} We show that both the Bellman operator $f$ and the policy evaluation operator $g^\pi$ satisfy the contractive, order preserving, and Lipschitz properties given in Definition~\ref{def:value_operator}. 
Contraction: given $(C,P) \in \mc{M}$, $g^\pi(\cdot, C, P)$ and $f(\cdot, C, P)$ are both $\gamma$-contractions~\cite[Prop.6.2.4]{puterman2014markov} on the complete metric space $(\reals^S, \norm{\cdot}_{\infty})$,
where $\gamma < 1$ is the discount factor. 

Order preservation: given $(C,P) \in \mc{M}$, the operator $g^\pi(\cdot, C, P)$ is order preserving~\cite[Lem.6.1.2]{puterman2014markov}. 
Consider $U,V\in\reals^S$ where $U\leq V$. If $g^\pi(\cdot, C, P)$ is order-preserving, $g^\pi(U, C,P)\leq g^\pi(V, C,P)$ for all $\pi\in \Pi$. Taking the infimum over $\Pi$, we have 
$f(U, C,P)=\inf_{\pi\in \Pi} g^\pi(U, C,P)\leq \inf_{\pi \in \Pi}g^\pi(V, C,P)=f(V, C,P)$.

$K(V)$-Lipschitz: given $(C,P), (C', P') \in \mc{M}$ and $V \in \reals^S$, \change{we prove the following for each $s\in[S]$,} 
\begin{multline}\label{eqn:lem_proof2_0}
|f_s(V, C', P')-f_s(V, C,P)|\\
\leq \norm{c_s' - c_s}_\infty + \gamma\norm{P'_s - P_s}_\infty\max\{\norm{\pi^\star_s}_\infty, \norm{\hat{\pi}_s}_\infty\}\norm{V}_\infty.  
\end{multline}
\change{We prove~\eqref{eqn:lem_proof2_0} for each of the following cases: 1) $f_s(V, C', P') \geq f_s(V, C,P)$, and 2) $f_s(V, C', P') \leq f_s(V, C,P)$. For case 1), l}et $\hat{\pi}$~\eqref{eqn:optimalPol} be the optimal policy for $f(V, C', P')$ and $\pi^\star$ be the optimal policy for $f(V, C,P)$. For $s \in [S]$, suppose $f_s(V, C', P') \geq f_s(V, C,P)$, then $0 \leq f_s(V, C', P')-f_s(V, C,P) \leq (c_s')^\top \hat{\pi}_s - c_s^\top \pi^\star_s + \gamma (P'_s\hat{\pi}_s)^\top V - \gamma (P_s{\pi}^\star_s)^\top V$. Since $\pi^\star$ is sub-optimal for $f(V, C',P')$, we can upper bound $|f_s(V, C', P')-f_s(V, C,P)| \leq (c_s' - c_s)^\top \pi^\star_s + \gamma[(P'_s - P_s){\pi_s}^\star]^\top V$. 
Since $\pi^\star_s, \hat{\pi}_s \in \Delta_A$, $\norm{\pi^\star_s}_\infty\leq 1$. \change{We conclude that~\eqref{eqn:lem_proof2_0} holds when $f_s(V, C', P') \geq f_s(V, C,P)$. For case 2), $f_s(V, C', P') \leq f_s(V, C,P)$,~\eqref{eqn:lem_proof2_0} also holds by similar arguments.}

\change{Letting $m' = (C',P')$ and $m = (C,P)$, w}e can upper bound $f(V,m) - f(V, m') = f - f'$ as 
\begin{align}
    \norm{f - f'}_\infty &\leq \max_{s\in[S]}\{\norm{c'_s - c_s}_\infty + \gamma \norm{(P_s - P'_s)^\top V}_\infty\} \\
    &\leq \max(1, \gamma\norm{V}_\infty)\norm{m - m'}_\infty.
\end{align}
The policy evaluation operator $g^\pi$ satisfies~\eqref{eqn:lem_proof2_0} if $\max\{\norm{\pi^\star_s}_\infty, \norm{\hat{\pi}_s}_\infty\}$ is replaced by $ \norm{{\pi}_s}_\infty$. Since $ \norm{{\pi}_s}_\infty \leq 1$, $g^\pi$ is $K(V)$-Lipschitz. \qed
\end{pf}
\end{document}